\documentclass{article}

\usepackage{microtype}
\usepackage{graphicx}
\usepackage{booktabs} 
\usepackage{mathtools}
\usepackage{enumitem}
\newcommand\scalemath[2]{\scalebox{#1}{\mbox{\ensuremath{\displaystyle #2}}}}
 \usepackage{multirow} 
\usepackage{amsmath}
\usepackage{amsthm}
\usepackage{color}
\usepackage{dsfont}
\usepackage{subcaption}
\usepackage{caption}
\usepackage{amssymb}

\DeclareMathOperator*{\argmin}{arg\,min}
\DeclareMathSymbol{\shortminus}{\mathbin}{AMSa}{"39}
\usepackage{soul}
\usepackage{tikz}
\usetikzlibrary{decorations.pathreplacing,positioning, arrows.meta}
\newtheorem{theorem}{Theorem}[section]

\newtheorem{proposition}{Proposition}
\usepackage{hyperref}

\usepackage[accepted]{icml2020}

\icmltitlerunning{Active Learning on Attributed Graphs via Graph Cognizant Logistic Regression and Preemptive Query Generation}

\begin{document}

\twocolumn[
\icmltitle{Active Learning on Attributed Graphs via Graph Cognizant Logistic Regression and Preemptive Query Generation}




\begin{icmlauthorlist}
\icmlauthor{Florence Regol$^{\dagger}$}{mc}
\icmlauthor{Soumyasundar Pal}{mc}
\icmlauthor{Yingxue Zhang}{hw}
\icmlauthor{Mark Coates}{mc}
\end{icmlauthorlist}
\icmlaffiliation{mc}{Department of Electrical and Computer Eningeering, McGill University, Montr{\'e}al, QC,  Canada.}
\icmlaffiliation{hw}{ Huawei Noah's Ark Lab, Montr\'{e}al Research Center, Montr\'{e}al, QC, Canada}

\icmlcorrespondingauthor{Florence Regol}{florence.robert-regol@mail.mcgill.ca}

\icmlkeywords{Machine Learning, ICML}

\vskip 0.3in
]



\printAffiliationsAndNotice{}  

\begin{abstract}
 Node classification in attributed graphs is an important task in multiple practical settings, but it can often be difficult or expensive to obtain labels. Active learning can improve the achieved classification performance for a given budget on the number of queried labels. The best existing methods are based on graph neural networks, but they often perform poorly unless a sizeable validation set of labelled nodes is available in order to choose good hyperparameters. We propose a novel graph-based active learning algorithm for the task of node classification in attributed graphs; our algorithm uses graph cognizant logistic regression, equivalent to a linearized graph-convolutional neural network
(GCN), for the prediction phase and maximizes the expected error
reduction in the query phase. To reduce the delay experienced by a labeller interacting with the system, we derive a preemptive querying system that calculates a new query during the labelling process, and to address the setting where learning starts with almost no labelled data, we also develop a hybrid algorithm that performs adaptive model averaging of label propagation and linearized GCN inference. We conduct experiments on five public benchmark datasets, demonstrating a
significant improvement over state-of-the-art approaches and
illustrate the practical value of the method by applying it to a
private microwave link network dataset.
\end{abstract}
\let\thefootnote\relax\footnotetext{$\dagger$Work done as intern at Huawei Noah's Ark Lab, Montr\'{e}al Research Center.}
\section{Introduction}

\noindent 
In many classification tasks there are explicit or implicit
relationships between the data points that need to be classified. One
can represent such data using a graph, where an edge between two nodes
(data points) indicates the presence of a relationship. The resultant
task of node classification has attracted significant attention from
the graph-learning research community, and numerous graph
learning architectures have been developed that yield impressive
performance, especially in semi-supervised
settings~\cite{defferrard2016,kipf2017,hamilton2017,velivckovic2018,zhuang2018,gao2018b,liu2019}. In such cases, knowledge of the graph topology can compensate for scarcity of labelled data.

In practice, the semi-supervised classification task often arises in
scenarios where it is challenging or expensive to obtain labels. If we
have the opportunity to decide which nodes to query, then we should
try to select the most informative nodes that lead to the best
classification accuracy. This is, in a nutshell,
the goal of active learning; as we acquire labels, we make decisions
about which label to query next based on what we have learned. This is
important in applications such as medical imaging, where generating
labels requires considerable valuable time from domain
experts~\cite{hoi2006, gal2017, kurzendorfer2017}. The development of active learning algorithms for node classification in graphs can be motivated by applications of graph convolutional neural networks (GNNs) in the medical field. For example, in~\cite{parisot2018}, GNNs are used to classify brain scan images, with the goal of predicting disease outcomes or detecting the presence of a disorder. Although we may have access to many brain scans and can specify relationships between them (thus building a graph), obtaining labels for them is expensive because it requires attention from medical experts. 

The early research that applies active learning to graph data mainly
focuses on the non-attributed graph setting~\cite{zhu2003b, ji2012,ma2013}. We focus on node classification for attributed graphs, so the recently
proposed GNN-based
methods~\cite{cai2017,gao2018a} are more aligned with the task we
address. The results reported in these works are usually based on GNNs
with hyperparameters that have been optimized using a large labelled
validation set. This is an unrealistic setting; if we have access to
such a large amount of labeled data, we should use much more of it to
train the classifier. As we illustrate in the experiments in
Section~\ref{sec:experiment}, if hyperparameters are not optimized,
but are chosen randomly from a reasonable range of candidate values,
the performance of the GNN-based active learning methods deteriorates
dramatically.

In this work, we aim to address the limitations of the GNN methods. We
propose an algorithm that is based on the Expected Error Minimization
(EEM) framework~\cite{settles2009}. In this framework, we select the
query that minimizes the expected classification error according to our current model. We use the simplified graph convolution (SGC)~\cite{wu2019}, a graph-cognizant logistic regression, as a
predictive model for the labels. This model, which can be derived as a
linearization of the graph convolutional neural network~\cite{wu2019},
performs much better when there is limited data compared to a GNN with
suboptimal hyperparameters, and achieves competitive accuracy as the
number of labels increases.

Most active learning techniques involve initial training of a model and then an iterative process of (i) identifying the best query by some criterion (the core step of active learning); (ii) obtaining the label from an oracle; and (iii) updating the model. In an interactive application, this can lead to a delay if the query generation of step (i) is not extremely fast. Although it is principled and competitive to other approaches, the EEM algorithm does have the disadvantage of an increased computational overhead. However, we note that a delay will also be introduced at step (ii); human labelling can take seconds (document categorization~\cite{settles2009}) to minutes (cancer diagnosis from skin lesion images~\cite{gal2017}, MRI tumour segmentation~\cite{kurzendorfer2017}, or fault detection in microwave link networks). With this in mind, the interactive delay can be reduced or even eliminated if the model update and query identification steps can be started and completed while the labelling is conducted. We develop such a preemptive strategy, based on a prediction of the labelling from the oracle.  

In summary, our paper makes the following contributions: (i) we
propose a practical approach for active learning in graphs that does
not have the unrealistic requirement of a validation set for
hyperparameter tuning; (ii) we extend the proposed approach to
introduce preemptive query generation in order to reduce or eliminate
the delay experienced by a labeller during interaction with the
system; (iii) we derive bounds on the error in risk evaluation associated with the preemptive prediction; (iv) we analyze performance on five public benchmark datasets and show a significant improvement compared to state-of-the-art GNN active learning methods (and label propagation strategies); (v) we illustrate the practical benefit of our method by demonstrating its application to a private, commercial dataset collected for the task of identifying faulty links in a microwave link network. 

\section{Related Research}

\subsection{Active learning on non-attributed graphs}

Many methods for active learning on graphs without node or edge
features are based on the idea of of propagating label information across the graph, and we hence refer to them throughout the paper as
label-propagation methods. The most successful techniques are
all based on the Binary Random Markov Field (BRMF) model. The model allows one to define a posterior on the unknown labels conditioned on the graph topology and the observed node labels. This model provides an effective mechanism for representing smoothness of labels with respect to the graph, but evaluating the posterior is a combinatorial problem. As a result, researchers have introduced relaxations or approximation strategies. \cite{zhu2003a} relax the BMRF to a Gaussian Random Field (GRF) model, and~\cite{zhu2003b, ji2012} and subsequently~\cite{ma2013} also employ this model to derive active learning methods. More recently,~\cite{berberidis2018} have applied the expected change maximization strategy to a GRF model. \cite{jun2016} takes
another approach by proposing a two-step approximation (TSA)
of the intractable combinatorial problem rather than relaxing the BMRF model.

These strategies offer the advantage of being principled methods that directly target the quantity we want to optimize, but label propagation-based models cannot take into account node features and consequently must rely on strong assumptions regarding the relationships between the graph topology and the data. Most label propagation methods struggle if the graph is not connected and do not usually translate well to an inductive setting, because query decisions rely on the knowledge of the complete graph topology.

\subsection{Graph neural network methods for active learning}

\cite{cai2017} leverage the output of a Graph Convolution Network
(GCN)~\cite{kipf2017} to design active learning metrics. Their method
is to alternate during the training of the GCN between adding one node
to the training set and performing one epoch of training. Selection of
the query node is based on a score that is a weighted mixture of three
metrics covering different active learning strategies: an uncertainty
metric, a density-based metric and a graph centrality metric. The
uncertainty metric is obtained by taking the entropy of the softmax
output given by the current GCN model. The density metric is based on the GCN node embeddings; the embeddings are clustered and the distance between each node's embedding and the centre of its cluster is computed. A more central embedding indicates a more representative
node. The graph centrality metric is independent of the GCN and only
relies on the position of the node in the graph. The weights change as
more nodes are added to the labelled set, in order to reflect the
increased confidence in the two metrics that are derived from the
output of the GCN. The weight adaptation schedule in~\cite{cai2017} is
fixed;~\cite{gao2018b} propose an alternative multi-armed bandit
algorithm that learns how to balance the contributions of the
different metrics. They argue that this mechanism can better adapt to the varying natures of different datasets.

\section{Problem Setting}
\subsection{Pool-based Formulation}
We consider the problem of active learning on an attributed graph
$\mathcal{G} = (\mathcal{V},\mathcal{E})$ for node classification using feature matrix $\mathbf{X}$ and
labels $\mathbf{Y}$.  The nodes are partitioned into two sets: a small initial
labelled set $\mathcal{L}_0$ from with node
labels $\mathbf{Y}_{\mathcal{L}_0}$ ($|\mathcal{L}_0|\ll |V|$), and a set
$\mathcal{U}_0$ consisting of the remaining unlabelled nodes. The
algorithm is given a budget of $b$ nodes that it can query from
$\mathcal{U}_0$ to augment $\mathbf{Y}_{\mathcal{L}_0}$. We denote by
$\mathcal{L}_t$ and $\mathcal{U}_t$ the sets of labelled and
unlabelled nodes, respectively, after $t$ nodes have been added to the
initial labelled set. The pool-based active learning formulation that
we are considering consists of three phases:

\begin{enumerate}[leftmargin=*]
    \item \textbf{Prediction Step} : $\mathbf{X}$, $\mathcal{G}$ and the current node labels,
      $\mathcal{L}_t$, are used to infer the labels of the
      nodes $\mathcal{V}$.
    \item \textbf{Query Step} : Until the budget is exhausted,
      select a node $q_t^* \in \mathcal{U}_t$ to query and
      to add to the labelled set $\mathcal{L}_t$.
\item \textbf{Labeling Step} : The oracle takes time $\Delta$ to label $q_t^*$. We update the sets: $\mathcal{U}_{t+1} = \mathcal{U}_t \setminus \{q_t^*\}$, $\mathcal{L}_{t+1} = \mathcal{L}_t \cup \{q_t^*\}$.
\end{enumerate}

 The goal is to select the best node $q_t^*$ to append to
$\mathcal{L}_t$ at each iteration $t$, in order to optimize the
prediction performance throughout the query process. We are not only
interested in the end result after exhausting the query budget, but
also in how quickly we can increase accuracy. Acquiring labels is
presumed to be expensive, so a solution that reaches
competitive performance with fewer nodes is desired.

In addition to the {\em transductive} setting outlined above, where we
know the entire graph and all attributes, we also consider an {\em
  inductive} setting, where our goal is to maximize performance over
an additional set of nodes $\mathcal{T}$; we know that these are
connected to the graph in some fashion, but we cannot query them and
we do not know the edges or features during the active learning
process.

\subsection{Reducing interaction delay}

With the phases outlined above, the active learning algorithm stalls
while waiting for the oracle to label $q_t^*$ at the third phase and then the oracle must
wait while the algorithm computes the best subsequent query node
$q_{t+1}^*$. This is inefficient and, for a human oracle, frustrating. We can address this by requiring the active learning
algorithm to identify the query node $q_t^*$ using the label set
$\mathcal{L}_{t-1}$. If the labelling time and the query computation
time are similar, then neither the oracle nor the algorithm stalls for
long. While the oracle is generating the label for $q_{t}^*$, this
{\em preemptive} active learning algorithm identifies in parallel the
best query node $q_{t+1}^*$ using
$\mathcal{L}_{t}$. Figure~\ref{fig:timeline} compares the timelines of
the standard single-query active learning procedure (query
generation algorithm waits for the oracle and vice versa) with the
preemptive strategy where labelling and query generation are performed
in parallel.

\begin{figure}[h]

\begin{subfigure}{\columnwidth}
  \begin{center}
  \begin{tikzpicture}[scale = 1]
\draw[thick, -Triangle] (0,0) -- (7cm,0) node[font=\scriptsize,below left=3pt and -8pt]{Time};

\foreach \x in {0,1.7,3.2,4.9}
\draw (\x cm,3pt) -- (\x cm,-3pt);

\foreach \x/\descr in {1.5/\text{$\nu$},3.5/\text{ $\nu+\Delta$},5/\text{$2\nu+\Delta$}}
\node[font=\scriptsize, text height=1.75ex,
text depth=.5ex] at (\x,-.3) {$\descr$};

\foreach \x/\perccol in
{3/100,4/75,5/0}
\draw[red, line width=4pt] 
(0,-.7) -- +(1.7,0);

\foreach \x/\perccol in
{1/100,2/75,3/25,5/0}
\draw[lightgray!\perccol!green, line width=4pt] 
(1.7,.5) -- +(1.5,0);

\foreach \x/\perccol in
{3/100,4/75,5/0}
\draw[red, line width=4pt] 
(3.2,-.7) -- +(1.7,0);

\foreach \x/\perccol in
{1/100,2/75,3/25,4/0}
\draw[lightgray!\perccol!green, line width=4pt] 
(4.9,.5) -- +(1.5,0);

\draw [thick ,decorate,decoration={brace,amplitude=5pt}] (1.7,0.7)  -- +(1.5,0) 
      node [black,midway,above=4pt, font=\scriptsize] {$\mathbf{Y}_{\mathcal{L}_t}$ query selection $\rightarrow  q^*_t$};
\draw [thick,decorate,decoration={brace,amplitude=5pt}] (4.9,-.9) -- +(-1.7,0)
      node [black,midway,font=\scriptsize, below=4pt] {Oracle labels $\mathbf{y}_{q^*_{t}}$.$ \mathbf{Y}_{\mathcal{L}_{t+1}} = [ \mathbf{Y}_{\mathcal{L}_t}, \mathbf{y}_{q^*_{t}} ] $};
\draw [thick,decorate,decoration={brace,amplitude=5pt}] (4.9,0.7) -- +(1.5,0)
      node [black,midway,above=4pt, font=\scriptsize] {$ \mathbf{Y}_{\mathcal{L}_{t+1}} \rightarrow q^*_{t+1}$};

\end{tikzpicture}
\end{center}
    \caption{Timeline for the standard active learning process.}
  \end{subfigure}
\centering
\begin{subfigure}{\columnwidth}
  \begin{center}
\begin{tikzpicture}[scale = 1]
\draw[thick, -Triangle] (0,0) -- (7cm,0) node[font=\scriptsize,below left=3pt and -8pt]{Time};

\foreach \x in {0,1.7,3.5,5.2}
\draw (\x cm,3pt) -- (\x cm,-3pt);

\foreach \x/\descr in {1.5/\text{$\nu$},3.5/\text{ $2\nu$},5.2/\text{$3\nu$}}
\node[font=\scriptsize, text height=1.75ex,
text depth=.5ex] at (\x,-.3) {$\descr$};


\foreach \x/\perccol in
{3/100,4/75,5/0}
\draw[red, line width=4pt] 
(0,-.7) -- +(1.7,0);
\foreach \x/\perccol in
{3/100,4/75,5/0}
\draw[red, line width=4pt] 
(1.75,-.7) -- +(1.7,0);
\foreach \x/\perccol in
{3/100,4/75,5/0}
\draw[red, line width=4pt] 
(3.5,-.7) -- +(1.7,0);
\foreach \x/\perccol in
{1/100,2/75,3/25,5/0}
\draw[lightgray!\perccol!green, line width=4pt] 
(0,.5) -- +(1.5,0);

\foreach \x/\perccol in
{1/100,2/75,3/25,4/0}
\draw[lightgray!\perccol!green, line width=4pt] 
(1.75,.5) -- +(1.5,0);
\foreach \x/\perccol in
{1/100,2/75,3/25,4/0}
\draw[lightgray!\perccol!green, line width=4pt] 
(3.5,.5) -- +(1.5,0);

\draw [thick ,decorate,decoration={brace,amplitude=5pt}] (0,0.7)  -- +(1.5,0) 
      node [black,midway,above=4pt, font=\scriptsize] {$\mathbf{Y}_{\mathcal{L}_{t-1}} \rightarrow \hat{q}^*_t$};
\draw [thick,decorate,decoration={brace,amplitude=5pt}] (3.5,-.9) -- +(-1.7,0)
      node [black,midway,font=\scriptsize, below=4pt] {Labeling process $\mathbf{y}_{\hat{q}^*_{t}}$};
\draw [thick,decorate,decoration={brace,amplitude=5pt}] (1.75,0.7) -- +(1.5,0)
      node [black,midway,above=4pt, font=\scriptsize] {$\mathbf{Y}_{\mathcal{L}_{t}} \rightarrow \hat{q}^{*}_{t+1}$};

  \end{tikzpicture}
  \end{center}
  \caption{Timeline for preemptive active learning process.}
\end{subfigure}
\caption{A comparison of the timelines of the standard single-query
  active learning process and the proposed preemptive process.}
\label{fig:timeline}
\end{figure}
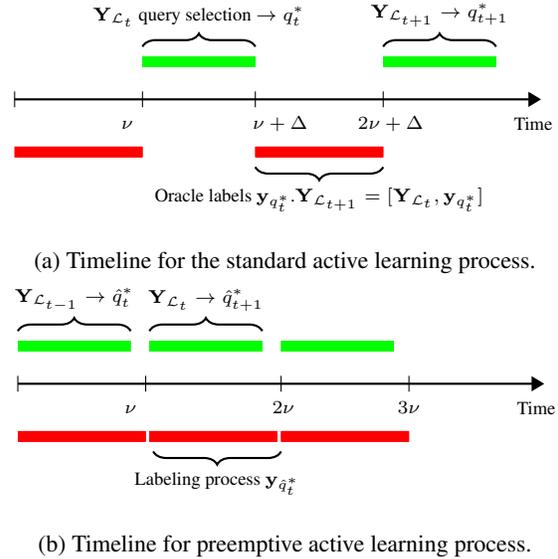

\section{Methodology}

\subsection{Expected error Minimization (EEM)}

An active learning algorithm based on error reduction
selects the query $q^*$ to minimize the expected error. For a classification task, the zero-one error is a suitable
choice. Denoting by $\mathcal{U}_t^{\shortminus q}$ the set of
unlabelled nodes after $t$ iterations of active learning with node $q$
removed, i.e., $\mathcal{U}_t^{\shortminus q} \triangleq
\mathcal{U}_t\setminus\{q\}$. The labels associated with this set are $\mathbf{Y}_{\mathcal{U}_t\setminus\{q\}}$.
Following~\cite{jun2016}, we can define
$R^{+q}_{|\mathbf{Y}_{\mathcal{L}_t}}$, the risk of adding node
$q \in \mathcal{U}_t$ given the current known label set
$\mathcal{L}_t$, as:
\begin{equation}
\scalemath{0.9}{
R^{+q}_{|\textbf{Y}_{\mathcal{L}_t}} \triangleq
 E_{\mathbf{y}_q}\Big[E_{\mathbf{Y}_{\mathcal{U}_t^{\shortminus q}}} \big[ \tfrac{1}{|\mathcal{U}_t^{\shortminus q}|} \sum_{i \in \mathcal{U}_t^{\shortminus q}} \mathds{1}[\hat{\mathbf{y}}_i \neq \mathbf{y}_i | \mathbf{y}_q, \mathbf{Y}_{\mathcal{L}_t}]\big]\Big]}
 \end{equation}
 Here $\hat{\mathbf{y}}_i$ is the label prediction at node $i$. We thus
 calculate expected error by summing error probabilities
 over the unlabelled set, minus the node $q$ we are considering.
 Define
 $\varphi^{+q}_{i,k, \mathbf{Y}_{\mathcal{L}_t}} \triangleq \Big(1- \displaystyle{
   \max_{k' \in K}}p(\mathbf{y}_i=k'|\mathbf{Y}_{\mathcal{L}_t}, \mathbf{y}_q=k)\Big)$, where $K$
 is the set of classes. If the query node
 $\mathbf{y}_q$ has label $k$, then $\varphi^{+q}_{i,k, \mathbf{Y}_{\mathcal{L}_t}}$
 represents the probability of making an error in the prediction
 $\hat{\mathbf{y}}_i$ of the label of node $i$. If we can
 compute the distribution $p(\mathbf{y} | \cdot )$, we can evaluate the risk of
 querying $q$:
\begin{equation}
  R^{+q}_{|\mathbf{Y}_{\mathcal{L}_t}} = \tfrac{1}{|\mathcal{U}_t^{\shortminus q}|}
  \sum_{k \in K} \sum_{i \in \mathcal{U}_t^{\shortminus q}} 
  \varphi^{+q}_{i,k, \mathbf{Y}_{\mathcal{L}_t}}  p(\mathbf{y}_q = k | \mathbf{Y}_{\mathcal{L}_t}) 
\end{equation}
The query algorithm selects the risk-minimizing node $q_t^*$:
\begin{equation}
    q_t^* = \argmin_{q \in \mathcal{U}_t} R^{+q}_{|\mathbf{Y}_{\mathcal{L}_t}}
\end{equation}
It remains to define the probabilistic model $p(\mathbf{y} | \cdot )$.

\subsection{Graph-cognizant logistic regression}

We propose to use a graph-cognizant logistic regression model to
obtain $p(\mathbf{y} | \cdot )$. Such a model was introduced by~\cite{wu2019},
where the SGC is derived as a simplified (linearized) version of the graph
convolutional network of~\cite{kipf2017}.  \cite{wu2019} showed that the simplified model can achieve competitive performance for a
significantly lower computational cost. In the EEM approach to active
learning, we must learn a new model for every potential query node, so
it is essential that the computational cost is relatively low. The
SGC meets our requirements: its
computational requirements are moderate and it takes into account the
graph structure and node features.

For a graph with adjacency matrix $\mathbf{A}$, let $\tilde{\mathbf{A}} = \mathbf{A}+\mathbf{I}$, $\mathbf{D}$ be
the degree matrix, and $\tilde{\mathbf{D}} = \mathbf{D}+\mathbf{I}$. We then define $\mathbf{S} \triangleq
\tilde{\mathbf{D}}^{-\frac{1}{2}}\tilde{\mathbf{A}}\tilde{\mathbf{D}}^{-\frac{1}{2}}$. This can
be interpreted as a degree normalized symmetrized adjacency matrix
(after self-loops have been added by the identity matrix $\mathbf{I}$). 
The prediction model has the form:
\begin{equation}
  \hat{\mathbf{Y}} =  \sigma(\mathbf{S}^{\ell}\mathbf{X}\mathbf{W})\,.
\end{equation}
We define $\tilde{\mathbf{X}}_{\mathcal{G}} \triangleq \mathbf{S}^{\ell}\mathbf{X}$; this can be interpreted as
graph-based preprocessing of node features. The parameter
$\ell$ controls the number of hops that are considered when
generating the final node representation. Usually using a 2-hop ($\ell=2$) neighborhood yields good results.

\subsection{Graph EEM (GEEM)}

Using the SGC model, we can compute a
risk for each query node. At each step $t$, we use the current known
labels $\mathbf{Y}_{\mathcal{L}_t}$ to find the weights
$\mathbf{W}_{ \mathbf{Y}_{\mathcal{L}_t}}$ by minimizing the error for predictions
$\hat{\mathbf{Y}}_{\mathcal{L}_t} = \sigma(\tilde{\mathbf{X}}_{\mathcal{L}_t}\mathbf{W}_{\mathbf{Y}_{\mathcal{L}_t}})$.
We use a standard iterative algorithm for maximum
likelihood logistic regression with L2 regularization (e.g., the
scikit-learn liblinear solver). We can then compute $p(\mathbf{y}_q = k | \mathbf{Y}_{\mathcal{L}_t}) = \sigma(\tilde{\mathbf{x}}_q\mathbf{W}_{\mathbf{Y}_{\mathcal{L}_t}})^{(k)}$, where the index $(k)$
indicates that we extract the $k$-th element of the vector. Then for each candidate node $q$, for each possible
class $k$, we solve:
\begin{equation}
  \hat{\mathbf{Y}}_{\mathcal{L}_t, +q,\mathbf{y}_k} = \sigma(\tilde{\mathbf{x}}_{\mathcal{L}_t, +q,\mathbf{y}_k}\mathbf{W}_{\mathbf{Y}_{\mathcal{L}_t}, +q,\mathbf{y}_k})\,.
\end{equation}
Here the notation $+q, \mathbf{y}_k$ indicates that we are adding node $q$ to
the labelled set and assigning it label $\mathbf{y}_k$.  For the adopted model,
$ \varphi^{+q}_{i,k, \mathbf{Y}_{\mathcal{L}_t}} = (1- \displaystyle{\max_{k' \in
    K} } \sigma(\tilde{\mathbf{x}}_{i}\mathbf{W}_{\mathbf{Y}_{\mathcal{L}_t}, +q,\mathbf{y}_k})^{(k')} )
$. The node to query is then the one that minimizes the risk:
\begin{equation}
\begin{split}
 q^*_t =  \argmin_{q \in \mathcal{U}_t} \tfrac{1}{|\mathcal{U}_t^{-q}|}   \sum_{k \in K}  \sum_{i \in \mathcal{U}^{\shortminus q}_t} \varphi^{+q}_{i,k, \mathbf{Y}_{\mathcal{L}_t}} \sigma(\tilde{\mathbf{x}}_q\mathbf{W}_{ \mathbf{Y}_{\mathcal{L}_t}})^{(k)}
\end{split}
\label{eq:eem}
\end{equation}

From this formulation, we can see that we first have to evaluate
$p(\mathbf{y} | \mathbf{Y}_{\mathcal{L}_t})$, then calculate $p(\mathbf{y} | \mathbf{Y}_{\mathcal{L}_t}, \mathbf{y} _q=k)$ for each of the $|K|$ potential augmented labelled sets $\mathcal{L}^{+q, \mathbf{y}_{k'}}$. 
This implies that we have a computational complexity of
$\mathcal{O}(|\mathcal{U}| |K| M)$, where $M$ represents the complexity
associated with training the model. For logistic regression,
this is the overhead involved in learning the weights
$\mathbf{W}_{\mathbf{Y}_{\mathcal{L}_t}}$. If the evaluation of $p(\mathbf{y} | \cdot )$ is
computationally expensive, then the time required to select a query
node can become prohibitive. This is a common disadvantage of
the expected reduction strategies~\cite{settles2009}. It then becomes apparent why using the
linearized version of the GCN is important.

The proposed algorithm requires the choice of very few
hyperparameters (only the number of hops $\ell$ and logistic regression parameters).
This contrasts with the active learning
approaches based on graph neural networks, where there are multiple hyperparameters that must be selected, and suboptimal choices can have a major impact on performance.

\subsection{Preemptive Query (PreGEEM)}

In many practical active learning scenarios, labelling is performed by a human, and it is often time-consuming; labelling a single data point (node) can take tens of seconds or minutes. It is desirable to have the next query identified as soon as the labeller has completed the labelling task. With the EEM algorithm formulated above, this is impossible, because the query identification in~\eqref{eq:eem} uses the label associated with the previous query node.

In this section, we outline an alternative approach that performs
pre-emptive query calculation, using the labelling time to identify
the next node to query. Instead of waiting for the oracle to label
$q^*_{t-1}$ to start the identification of $q^*_{t}$ during iteration
$t$, we propose to approximate the risk before knowing
$\mathbf{y}_{q^*_{t-1}}$. The direct approach is to replace the risk
$R^{+q_t}_{|\mathbf{Y}_{\mathcal{L}_t}}$ with the expectation over the possible
values of $\mathbf{y}_{q^*_{t-1}}$, but this increases the computational
complexity by a factor of
 $|K|$, which is highly undesirable. To avoid this penalty, we further approximate
 $R^{+q}_{|\mathbf{Y}_{\mathcal{L}_{t-1}}, q^*_{t-1}}$ using the value of risk
 for the label at
 the mode of $p(\mathbf{y}_{q_{t-1}^*}| \mathbf{Y}_{\mathcal{L}_{t-1}})$. Effectively, we use the
 predicted label $\hat{\mathbf{y}}_{q^*_{t-1}}$ of the previous model
 $p(\cdot | \mathbf{Y}_{\mathcal{L}_{t-1}})$ to form an augmented labelled set
$\mathbf{Y}'_{\mathcal{L}_{t}} =  \{\mathbf{Y}_{\mathcal{L}_{t-1}} \cup \{ \hat{\mathbf{y}}_{q^*_{t-1}} \}\}$ and define an approximate risk:
 \begin{equation}
 \scalemath{0.9}{
\hat{R}^{+q}_{|\mathbf{Y}'_{\mathcal{L}_{t}}} \triangleq 
 E_{\mathbf{y}_q}\Big[E_{\mathbf{Y}_{\mathcal{U}_t^{\shortminus q}}} \big[ \tfrac{1}{|\mathcal{U}_t^{\shortminus q}|} \sum_{i \in \mathcal{U}_t^{\shortminus q}} \mathds{1}[\hat{\mathbf{y}}_i \neq \mathbf{y}_i | \mathbf{y}_q, \mathbf{Y}'_{\mathcal{L}_t}]\big]\Big]}\,.
 \end{equation}
The query node is then $\hat{q}_t^* =  \displaystyle{\argmin_{q \in
  \mathcal{U}_t}}\hat{R}^{+q}_{|\mathbf{Y}_{\mathcal{L}'_{t}}}$. We call this new approach the Preemptive Graph EEM (PreGEEM).

 Figure~\ref{fig:risk} compares the evolution of the two active learning algorithms GEEM and PreGEEM for a small subset of nodes to illustrate how using the approximated risk can impact the query process. We see that the evaluated risks are similar, and although the ordering of query nodes differs, after five steps the same nodes have been selected by both algorithms.

\begin{figure}
    \centering
    \includegraphics[scale=0.5]{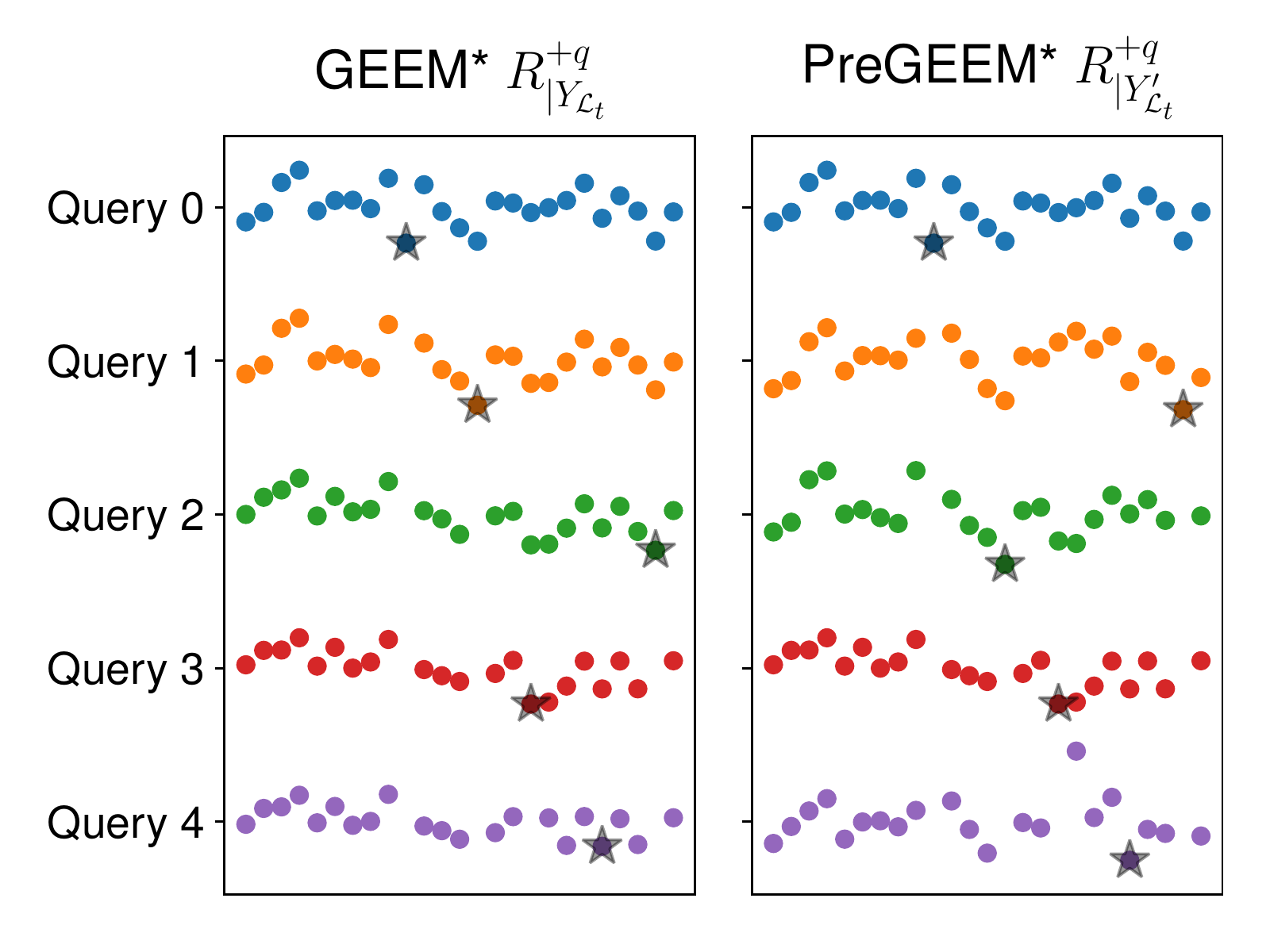}
    \caption{Risk comparison for GEEM vs PreGEEM. This follows the risk computations for 25 nodes in cora dataset for one trial. The black star indicates which node was selected, (following the algorithm, it is the one with the lowest expected risk.)}
    \label{fig:risk}
\end{figure}

\subsection{Bounds on the PreGEEM Risk Error}

We now present bounds on the risk estimation error that can arise by using a predicted label. We focus on the one-step error
$|R^{+q}_{|\mathbf{Y}'_{\mathcal{L}_{t}}} - R^{+q}_{|\mathbf{Y}_{\mathcal{L}}}|$, where
the labelled sets differ for only one label. For clarity, we first derive a bound for the binary classification task. We then state the more general bound for multiclass classification. Complete proofs are provided in the supplementary material.

The following proposition, which follows straightforwardly from
Theorem 1 and Lemma 1 in~\cite{sivan2019}, bounds the difference in
prediction values for two label sets that differ for $m$ data points. The bound is expressed in terms of the regularization
parameter $\lambda$ and the graph pre-processed feature
vectors $\tilde{\mathbf{x}}_i$.
\begin{proposition}
For weights $\mathbf{w}_1$ and $\mathbf{w}_2$ derived by $L2$-penalized
logistic regression to two datasets with common
feature vectors but label sets differing for $m$ vectors indexed by
$\{i_{k}\}, k=1,\dots,m$, define $\eta \triangleq \tfrac{1}{2\lambda}
{\sum_{k=1}^m}  \lvert \lvert \tilde{\mathbf{x}}_{i_k}\rvert \rvert$ and
$b_{\pm \eta}(\mathbf{w}_2,i) \triangleq \sigma(\mathbf{w}_2^{\top}\mathbf{\tilde{x}}_i) - \sigma(
\mathbf{w}_2^{\top}\mathbf{\tilde{x}}_j \pm 2\eta\lvert \lvert \mathbf{\tilde{x}}_i \rvert \rvert)$.
For any $\mathbf{\tilde{x}}_i$, $  \lvert  \sigma(\mathbf{w}_1^{\top}\mathbf{\tilde{x}}_i) - \sigma(\mathbf{w}_2^{\top}\mathbf{\tilde{x}}_i)\rvert
  \leq \max(\lvert b_{\eta}(\mathbf{w}_2,i) \rvert,\lvert b_{-\eta}(\mathbf{w}_2,i) \rvert)$.
\label{prop_logistic_diff}
\end{proposition}  
The following result characterizes the risk error when we perform 
$L2$-regularized binary logistic regression on
$\{\mathbf{Y}'_{\mathcal{L}_t}\cup \{\mathbf{y}_q=k\}\}$ to derive
weights $\mathbf{w}_{2,k}$ for $k\in\{0,1\}$. Define $\eta_q \triangleq \tfrac{1}{2\lambda}\left( \lvert\lvert \mathbf{\tilde{x}}_q \rvert\rvert
+  \lvert\lvert \mathbf{\tilde{x}}_{q^*_{t-1}} \rvert\rvert \right)$, and
$\tilde{b}(\eta_q,\mathbf{w}_2,i) =
\displaystyle{\max_{k\in\{0,1\}}}\max\left(|b_{+\eta_q}(\mathbf{w}_{2,k},i)|,
  |b_{-\eta_q}(\mathbf{w}_{2,k},i)|\right)$.
\begin{theorem}
The risk error arising from applying binary L2-regularized logistic
regression with regularization parameter $\lambda$ to two labelled datasets
$\mathbf{Y}_{\mathcal{L}_t}$ and $\mathbf{Y}'_{\mathcal{L}_t}$ that differ by one label,
associated with the node $q^*_{t-1}$, is bounded as:
\begin{equation}
  \lvert R^{+q}_{|\mathbf{Y}_{\mathcal{L}_t}} -
  R^{+q}_{|\mathbf{Y}'_{\mathcal{L}_t}}\rvert \leq
  \dfrac{1}{|\mathcal{U}_t^{\shortminus q}|} \displaystyle{\sum_{i \in
      \mathcal{U}_t^{\shortminus q}}} \tilde{b}(\eta_q,\mathbf{w}_2,i)\,.\nonumber
\end{equation}
\label{thm:binary}
\end{theorem}

\begin{proof}[Sketch of Proof]
  We define the random variable $\varphi^{+q}_{i, \mathbf{Y}_{\mathcal{L}_t}}$
  which takes value $\varphi^{+q}_{i, k, \mathbf{Y}_{\mathcal{L}_t}}$ with
  probability $p(\mathbf{y}_q=k|\mathbf{Y}_{\mathcal{L}_t})$ for $k\in\{0,1\}$.
The difference in risk is:
\begin{align}
  \tfrac{1}{|\mathcal{U}_t^{\shortminus q}|}\sum_{i \in \mathcal{U}_t^{\shortminus q}} 
  E_{\mathbf{y}_q}[\varphi^{+q}_{i, \mathbf{Y}_{\mathcal{L}_t}}] -
  E_{\mathbf{y}_q} [\varphi^{+q}_{i,\mathbf{Y}'_{\mathcal{L}_t}}]\,, \label{eq:diffrisk}
\end{align}
where $E_{\mathbf{y}_q}$ denotes expectation over $\mathbf{y}_q$
conditioned on the observed label set, either $\mathbf{Y}_{\mathcal{L}_t}$ or $\mathbf{Y}'_{\mathcal{L}_t}$.
For query node $q$, for each $k_1, k_2 \in \{0,1\}$, we learn
weights $\mathbf{w}_{1,k_1}$ and $\mathbf{w}_{2,k_2}$ using $\{\mathbf{Y}_{\mathcal{L}_t} \cup
\{\mathbf{y}_q = k_1\}\}$ and $\{\mathbf{Y}'_{\mathcal{L}_t} \cup \{\mathbf{y}_q = k_2\}\}$,
respectively. For each $i \in \mathcal{U}_t^{\shortminus q}$, we have:
\begin{align}
\lvert \varphi^{+q}_{i,k_1, \mathbf{Y}_{\mathcal{L}_t}} - \varphi^{+q}_{i,k_2, \mathbf{Y}'_{\mathcal{L}_t}} \rvert &\leq \lvert \sigma(\mathbf{w}_{1,k_1}^{\top}\mathbf{\tilde{x}}_i) - \sigma(\mathbf{w}_{2,k_2}^{\top}\mathbf{\tilde{x}}_i)\rvert \,,\nonumber\\
&\leq \tilde{b}(\eta_q,\mathbf{w}_2,i)
\,.\label{ineq:phi_diff}
\end{align}
Here the first inequality follows from the definition of
$\varphi^{+q}_{i,k_1, \mathbf{Y}_{\mathcal{L}_t}}$ and the property that for
$0 \leq p_1, p_2 \leq 1$, 
$\lvert\min(p_1,1-p_1)-\min(p_2,1-p_2)\rvert \leq \lvert
p_1-p_2\rvert$. The second inequality follows from
Proposition~\eqref{prop_logistic_diff}, observing that the labels can
differ only for nodes $q^*_{t-1}$ and $q$.

Observe that for random variables $X$ and $Y$ taking values in $[a,b]$
and $[c,d]$, respectively, 
$\lvert E_{X}[X] - E_{Y}[Y] \rvert \leq \max \Big(
\lvert a-d\rvert, \lvert b-c\rvert\Big)$. Applying this
to~\eqref{eq:diffrisk} and then employing~\eqref{ineq:phi_diff} leads
to the stated bound on the risk error.

\end{proof}
The following bound applies for the case of multiclass
classification. For a given label set $\mathbf{Y}$, we learn weights $\mathbf{w}^{(k)}$ for each class
$k \in \{1,\dots,K\}$ using L2-regularized binary one-vs-all logistic
regression. The output prior to normalization for a given feature vector
$\mathbf{\tilde{x}}_i$ is then $\sigma(\mathbf{w}^{(k)\top}\mathbf{\tilde{x}}_i)$. We then normalize
by dividing by $C_i(\mathbf{w}) = \sum_{k=1}^K \sigma(\mathbf{w}^{(k)\top}\mathbf{\tilde{x}}_i)$ to
obtain a probability vector.
Let $\mathbf{w}_{2,k}^{(k')}$ be the weight
vector learned for class $k'$ using label data
$\{\mathbf{Y}'_{\mathcal{L}_t}\cup \{\mathbf{y}_q = k\}\}$. Let $\rho(\mathbf{w}_{2,k},\eta,i) \triangleq \displaystyle{\max_{k' \in K}}\max(\lvert b_{\eta}(\mathbf{w}^{(k')}_{2,k},i) \rvert,\lvert
   b_{-\eta}(\mathbf{w}^{(k')}_{2,k},i) \rvert)$ for $b_{\pm \eta}(\mathbf{w}_2,i) \triangleq \sigma(\mathbf{w}_2^{\top}\mathbf{\tilde{x}}_i) - \sigma(
\mathbf{w}_2^{\top}\mathbf{\tilde{x}}_i \pm 2\eta\lvert \lvert \mathbf{\tilde{x}}_i \rvert \rvert)$. 
Define 
\begin{align*}
\tilde{\beta}(\mathbf{w}_2,\eta,i) \triangleq \max_{k,k'}
&\left(\left\lvert \tfrac{\sigma(\mathbf{w}^{(k')\top}_{2,k}\mathbf{\tilde{x}}_i)}{C_{i}(\mathbf{w}_{2,k})} -\tfrac{\mathbf{w}^{(k')\top}_{2,k}\mathbf{\tilde{x}}_i-\rho(\mathbf{w}_{2,k},\eta,i)}{C_{i}(\mathbf{w}_{2,k})
                                           +4\rho(\mathbf{w}_{2,k},\eta,i)}\right\rvert,
  \right. \nonumber\\
 &\quad\left. \left\lvert\tfrac{\sigma(\mathbf{w}^{(k')\top}_{2,k}\mathbf{\tilde{x}}_i)}{C_{i}(\mathbf{w}_{2,k})} -
        \tfrac{\mathbf{w}^{(k')\top}_2\mathbf{\tilde{x}}_i+\rho(\mathbf{w}_2,\eta,i)}{C_{i}(\mathbf{w}_{2,k})
                                                          -4\rho(\mathbf{w}_2,\eta,i)}\right\rvert \right)\,.
\end{align*}                                                          
\begin{theorem}
The risk error arising from multiclass regression performed via
repeated one-vs-all L2-regularized logistic
regression with regularization parameter $\lambda$ to labelled datasets
$\mathbf{Y}_{\mathcal{L}_t}$ and $\mathbf{Y}'_{\mathcal{L}_t}$ that differ by one label,
associated with the node $q^*_{t-1}$, is bounded as:
\begin{equation}
\lvert R^{+q}_{|\mathbf{Y}_{\mathcal{L}_t}} -
R^{+q}_{|\mathbf{Y}'_{\mathcal{L}_t}}\rvert \leq
\dfrac{1}{|\mathcal{U}_t^{\shortminus q}|}
\displaystyle{\sum_{i \in \mathcal{U}_t^{\shortminus q}}}\tilde{\beta}(\mathbf{w}_2,\eta_q,i) \,,
\end{equation}
for $\eta_q \triangleq \tfrac{1}{2\lambda}\big( \lvert\lvert \mathbf{\tilde{x}}_q \rvert\rvert
+  \lvert\lvert \mathbf{\tilde{x}}_{q^*_{t-1}} \rvert\rvert \big)$.
\end{theorem}
The proof is similar to that of Theorem~\ref{thm:binary},
but more involved, and is provided in the supplementary material.

\subsection{Combined method}
The most extreme case of active learning is when we start with only
one labelled node. In this scenario, the logistic regression model
cannot make useful predictions until at least some nodes have been
queried. To address this scenario, we combine our algorithm with
a label-propagation method. The aim is to first use label-propagation
when very few node labels are available, then switch to a combination
of both algorithms when more information is available, and finally
transition to the more accurate graph-cognizant logistic regression.
Bayesian model averaging provides a mechanism to make this transition~\cite{minka2002}. 

In Bayesian model averaging, we have $V$ different
classifiers and our belief is that one of these models is correct. We
start with a prior $p(v)$ over each model. After observing data
$D$, we compute the model evidence $p(D|v)$.
Using Bayes' rule, we can compute the posterior $p(v|D) =
p(v)p(D|v)/p(D)$ and then weight the predictions:
\begin{align}
p(y_i|D) = \sum_{v=1}^V p(y_i,v|
  D)  = \sum_{v=1}^V p(y_i|v, D)p(v|D). \nonumber
\end{align}
In the context of expected error minimization active learning, we need
to evaluate the risk associated with a query. We introduce a
model-dependent zero-one risk, and
in our combined method, we employ a model-averaged risk: 
\begin{equation}
R^{+q}_{|\mathbf{Y}_{\mathcal{L}}} = \sum_{v=1}^V  E_{\mathbf{y}_q}\Big[R(\mathbf{Y}_{\mathcal{L}}, \mathbf{y}_q, v)\Big]
p(v|\mathbf{Y}_{\mathcal{L}})\,.
\end{equation}
In order to compute (or approximate) this expression, we need to
evaluate $p(v|\mathbf{Y}_{\mathcal{L}})$. Assuming that we
have equal prior belief in the models available to us, this is
equivalent to calculating the marginal likelihood $p(\mathbf{Y}_{\mathcal{L}}|v)$. 

We incorporate two models, one
based on label propagation and the other based on logistic regression.
For the binary random field model that underpins the label propagation
classifiers, there are no learnable model parameters (there is one
fixed hyperparameter). Evaluating the
evidence $p(\mathbf{Y}_{\mathcal{L}}|v)$ is thus equivalent to computing $p(\mathbf{Y}_{\mathcal{L}})$ under
the BMRF model. This is a combinatorial problem, but we can factorize the joint probability into a chain rule of conditionals and use the
same two-stage approximation (TSA) that is employed
in~\cite{jun2016}. Additional details are provided in the supplementary material. We denote this evidence
approximation by $\lambda_{TSA,
  \mathbf{Y}_{\mathcal{L}}} \triangleq p(\mathbf{Y}_{\mathcal{L}}|BMRF)$.

For the logistic regression model, we are using $p(\mathbf{y}_i = k|v,\mathbf{W})=\sigma(\mathbf{W}^{\top}\mathbf{\tilde{x}}_i)^{(k)}$.
The joint probability of the complete labelled set is then evaluated as $p(\mathbf{Y}_{\mathcal{L}}|v,\mathbf{W}) = \prod_{\mathbf{y}_i \in \mathbf{Y}_{\mathcal{L}}}
\sigma(\mathbf{W}^{\top}\mathbf{\tilde{x}})^{(k_i)}$, where $k_i$ is the categorical index of $\mathbf{y}_i$.
We then have 
$p(\mathbf{Y}_{\mathcal{L}}|v) =
\int_\Theta \prod_{\mathbf{y}_i \in \mathbf{Y}_{\mathcal{L}}}\sigma(\mathbf{W}^{\top}\mathbf{\tilde{x}})^{(k_i)}p(\mathbf{W})\,d\mathbf{W}$.
To calculate the
evidence, we thus need to integrate over the weight matrix $\mathbf{W}$, which
is not analytically tractable. We choose to approximate $p(\mathbf{Y}_{\mathcal{L}}|v)$ by
$p(\mathbf{Y}_{\mathcal{L}}|v,\mathbf{W}_{\mathbf{Y}_\mathcal{L}})$, and we denote this as $\lambda_{LG, \mathbf{Y}_{\mathcal{L}}}\triangleq\prod_{\mathbf{y}_i \in \mathbf{Y}_{\mathcal{L}}}
\sigma(\mathbf{W}_{\mathbf{Y}_\mathcal{L}}^{\top}\mathbf{\tilde{x}})^{(k_i)}$. This leads to a sufficiently accurate approximation
of the evidence for our purpose (which is just to achieve an adaptive
balance between label propagation and SGC).

We use the TSA algorithm~\cite{jun2016} as the label propagation-based
estimator, leading to a new combined approach for selecting the query
node. We solve
\begin{align*}
\argmin_{q \in \mathcal{U}_t} R^{+q}_{|\mathbf{Y}_{\mathcal{L}}} = \bar{\lambda}_{LG, \mathbf{Y}_{\mathcal{L}}}
  R^{+q}_{|\mathbf{Y}_{\mathcal{L}}, LG} + \bar{\lambda}_{TSA,\mathbf{Y}_{\mathcal{L}}}, 
  R^{+q}_{|\mathbf{Y}_{\mathcal{L}}, TSA} \,,
\end{align*} 
where $\bar{\lambda}_{LG, \mathbf{Y}_{\mathcal{L}}}$ and $\bar{\lambda}_{TSA, \mathbf{Y}_{\mathcal{L}}}$ are our normalized estimates for the model evidences for the SGC model and the TSA label propagation model, respectively, after observing the data $\mathbf{Y}_{\mathcal{L}}$.

\begin{figure*} [ht]
\centering
\begin{tabular}{ccc}
  \includegraphics[scale=0.35]{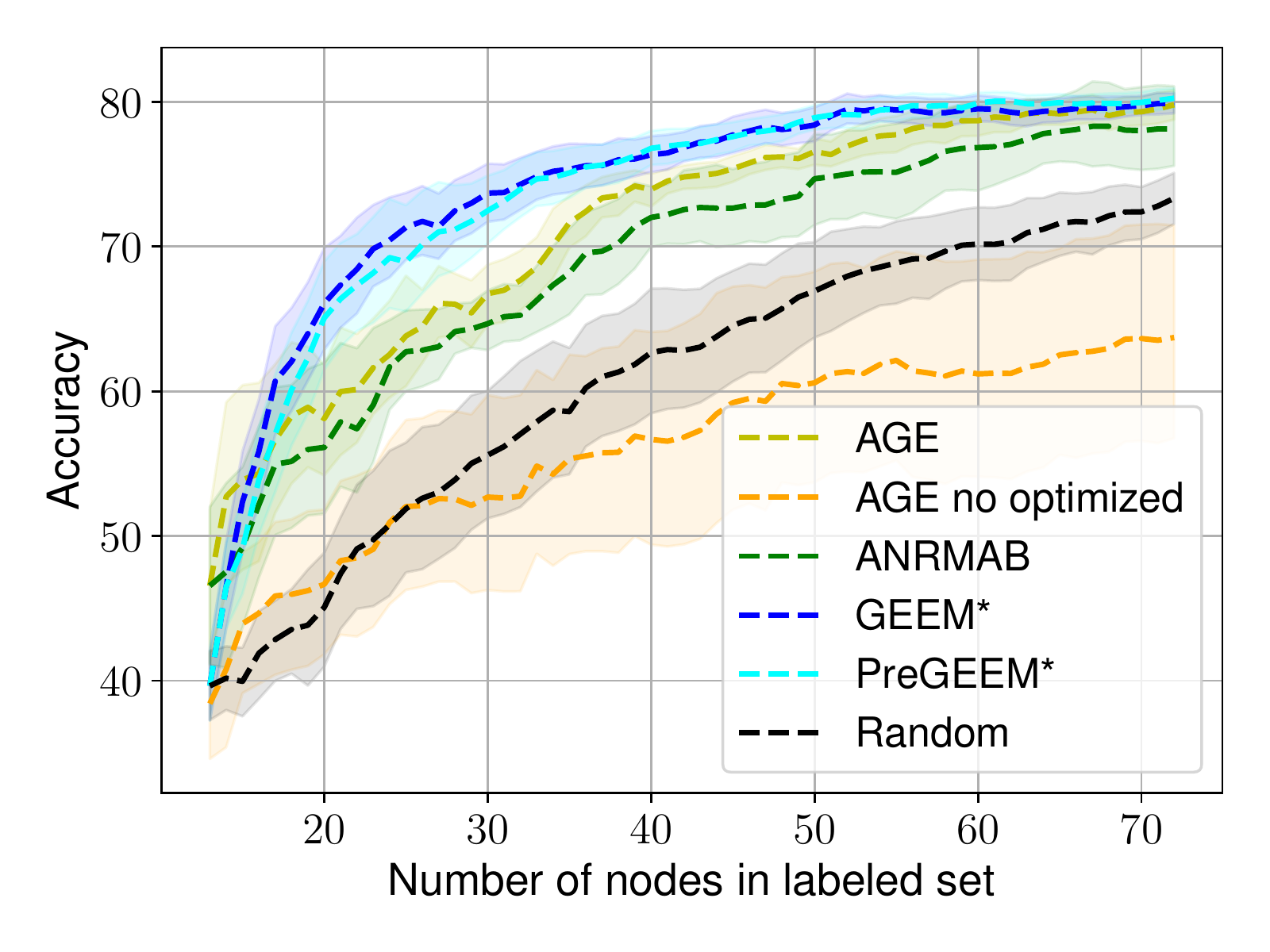} & \hspace{-2em}
             \includegraphics[scale=0.35]{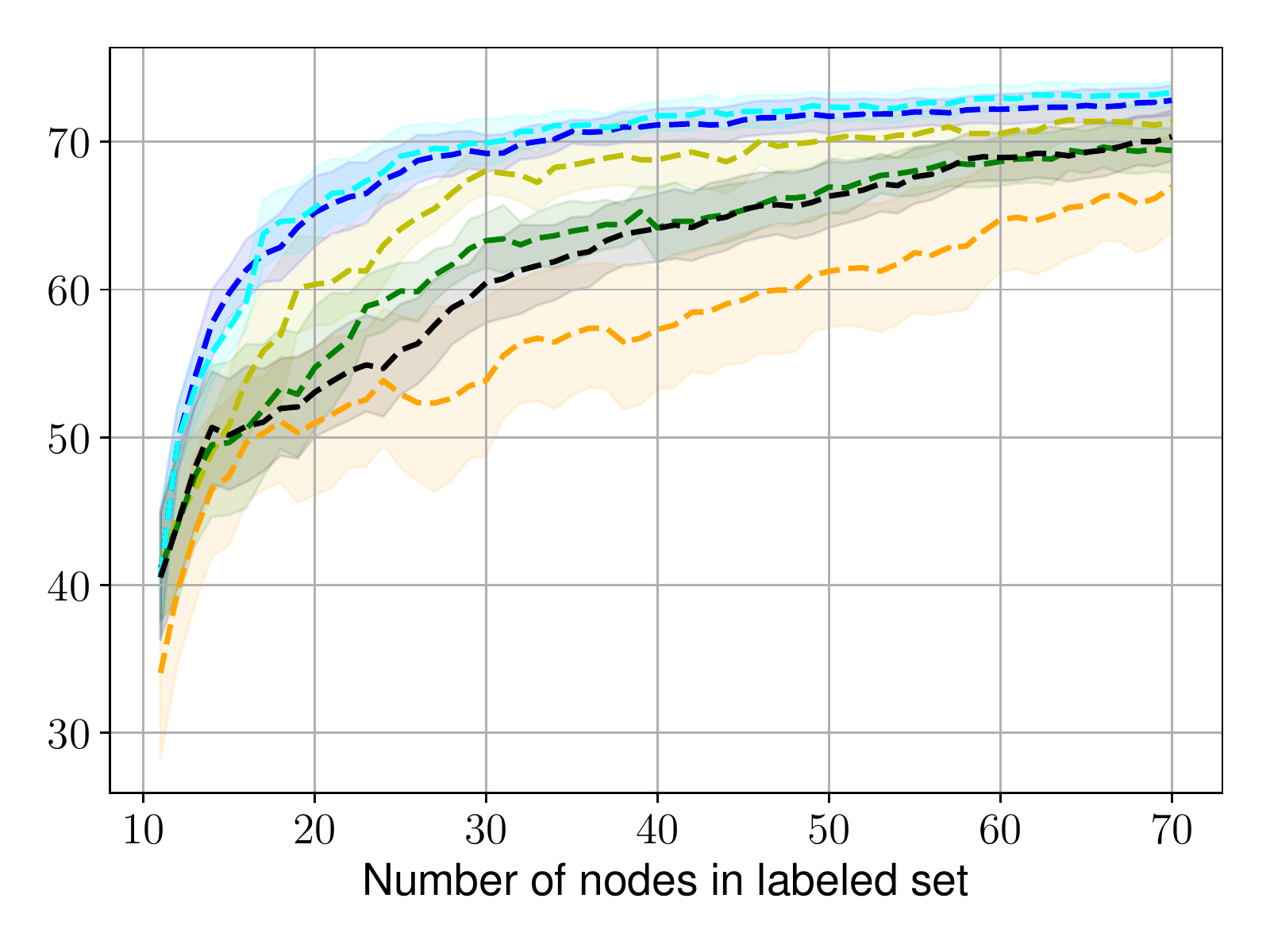} & \hspace{-2em}
        \includegraphics[scale=0.35]{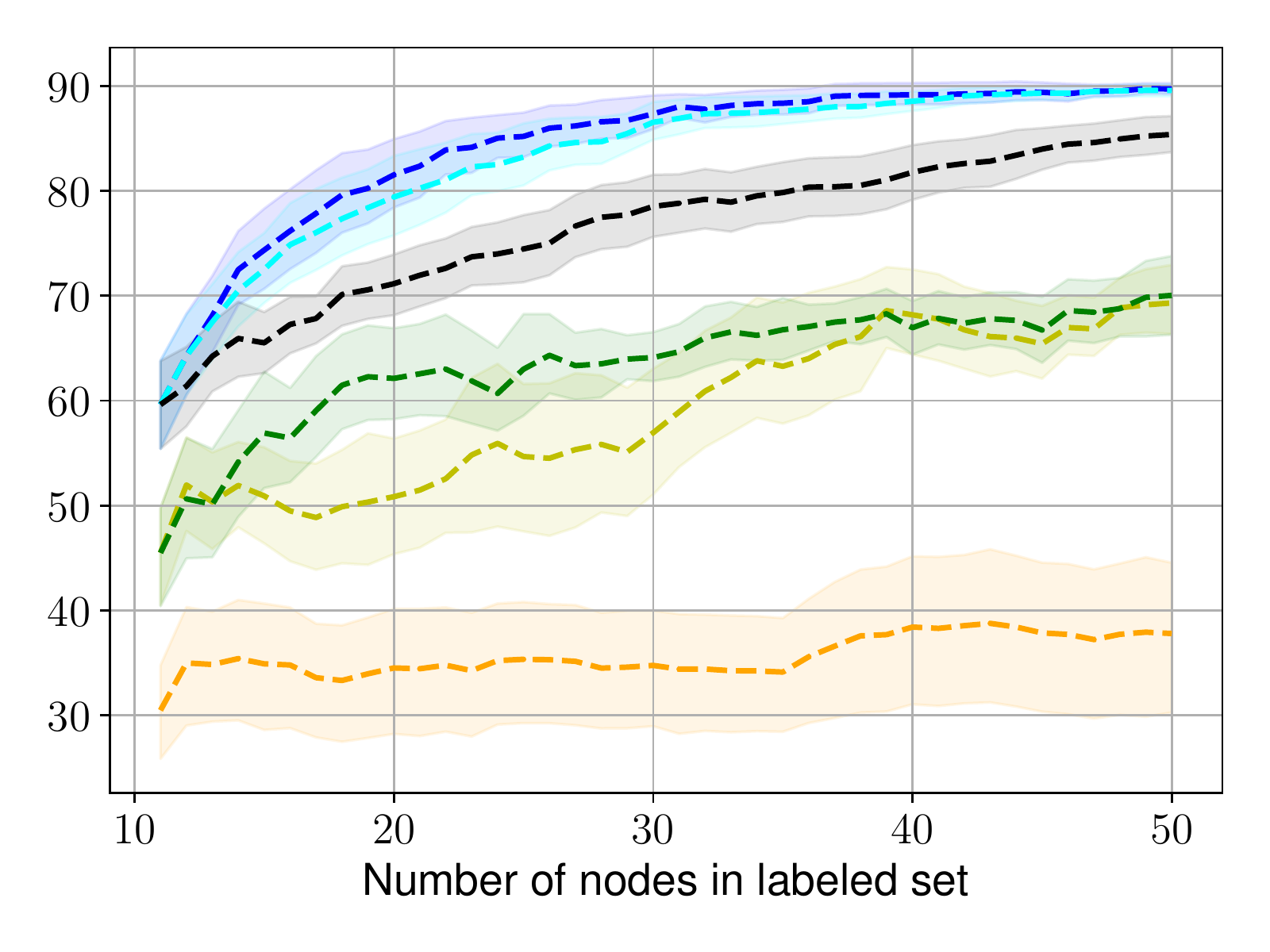} \\
  \text{\small{(a) Cora}}  & \text{\small{(b) Citeseer}  } & \text{\small{(c) Amazon-photo}  }  \\[12pt]
\end{tabular}
\caption{Comparison of performance of active learning algorithms for
  Experiment 1. Each point on a curve shows the mean classification
  accuracy achieved across 20 random partitions after the labelled set has been expanded to the
  indicated number of nodes. The shaded regions indicate 5/95
  confidence intervals on the means derived using bootstrap.}
\label{fig:res}
\end{figure*}

\section{Experiments}~\label{sec:experiment}
We examine performance using five of the node classification
benchmarks in~\cite{shchur2018}. Cora, Citeseer~\cite{sen2008} and Pubmed~\cite{namata2012} are citation datasets. Nodes represent journal articles and an undirected edge is included when one article cites another. The node features are bag-of-words representations of article content. Amazon-Photo and Amazon-Computers are graphs based on customers' co-purchase history records. For each dataset we isolate the largest connected component in the graph following~\cite{shchur2018}. The description of
the dataset statistics is shown in
Table~\ref{table:dataset_statistics}.

\begin{table}[htbp]
  \caption{Statistics of evaluation datasets.}
  \vspace{0.1cm}
  \label{table:dataset_statistics}
\centering
\resizebox{\columnwidth}{!}{%

\begin{tabular}{lccccc}
\toprule
\textbf{Dataset} & \textbf{Classe}s &\textbf{Features} & \textbf{Nodes} & \textbf{Edges} & \multicolumn{1}{c}{\textbf{\begin{tabular}[c]{@{}c@{}}Edge \\ Density\end{tabular}}}  \\ \midrule
\textbf{Cora} & 7 & 1,433 & 2,485 & 5,069 & 0.04\% \\ 
\textbf{Citeseer} & 6 & 3,703 & 2,110 &  3,668   & 0.04\% \\ 
\textbf{Pubmed} & 3 & 500& 19,717 &  44,324   & 0.01\% \\ 
\textbf{Am-comp.} & 10 & 767 & 13,381 &   245,778  & 0.07\% \\ 
\textbf{Am-photo} & 8 & 745 & 7,487 &   119,043  & 0.11\% \\ 
 \textbf{Microwave} & 2 & 19 & 322 &  5,753 &  5.54\% \\ \bottomrule
\end{tabular}
}
\label{data}
\end{table}

\subsection{Baselines} We compare the following active learning
algorithms: \textbf{(i) Random}: This baseline chooses a node to query
by uniform random selection, and then performs classification using SGC; \textbf{(ii) AGE}: The graph neural network based
algorithm proposed by~\cite{cai2017}; \textbf{(iii) ANRMAB}: The graph
neural network algorithm proposed by~\cite{gao2018b}, in which a
multi-arm bandit is used to adapt the weights assigned to the
different metrics used when constructing the score to choose a query
node; \textbf{(iv) TSA}: The label-propagation algorithm based on a
two-stage approximation of the BMRF model~\cite{jun2016}; \textbf{(v)
  EC-TV, EC-MSD}: The label-propagation algorithms based on a Gaussian
random field approximation to the BMRF model~\cite{berberidis2018};
\textbf{(vi) GEEM}: The
proposed algorithm based on SGC and
expected error minimization; \textbf{(vii) PreGEEM}: The proposed
algorithm with preemptive queries; \textbf{(viii) Combined} : The
proposed combined algorithm that uses Bayesian model averaging to adaptively
merge SGC and label propagation in an
EEM framework.

\subsection{Experimental Settings} For each experiment, we
report the average over 20 trials with different random
partitions. All GCNs and SGCs have 2 layers. The weight-adapting parameter of AGE is set
to the values in~\cite{cai2017} and to 0.995 for non-included
datasets.  For the larger datasets, Am-Photo, Am-Comp and Pubmed, we reduce
computational complexity for GEEM and PreGEEM by evaluating risk using
a subset of 500 nodes, selected randomly in an approach similar
to~\cite{roy2001}. This has minimal impact on performance. The GCN
hyperparameters are set to the values found by~\cite{shchur2018} to be
the best performing hyperparameter configurations. Early
stopping is not employed because access to a validation set is not a
reasonable assumption in an active learning setting. We also include a
``non-optimized'' version of AGE; this is because, in practice, we
would usually not have access to the tuned hyperparameters provided
by~\cite{shchur2018}, because these are derived using a large
validation set. For the non-optimized version of AGE, the
hyperparameter configuration for each trial was randomly selected from
the values considered in the grid search of~\cite{shchur2018}.
\begin{figure*}[ht]
  \par\medskip
  \begin{subfigure}[t]{0.24\textwidth}
  \includegraphics[scale=0.27]{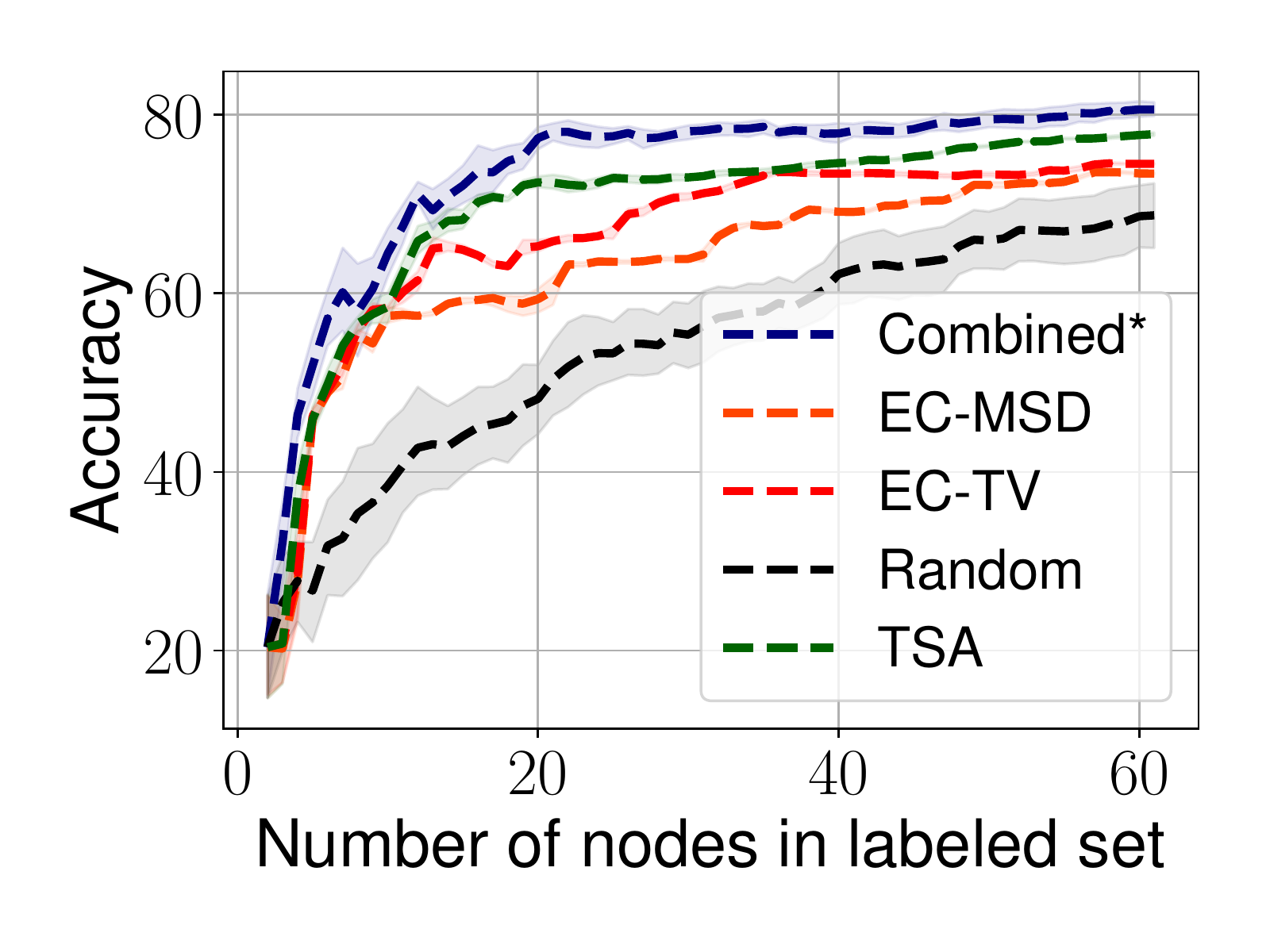}
  \vspace*{-5mm} \centering \caption{Cora - transductive}
    \label{fig:coratransdu}
    \end{subfigure}\hfill
  \begin{subfigure}[t]{0.24\textwidth}
\includegraphics[scale=0.27]{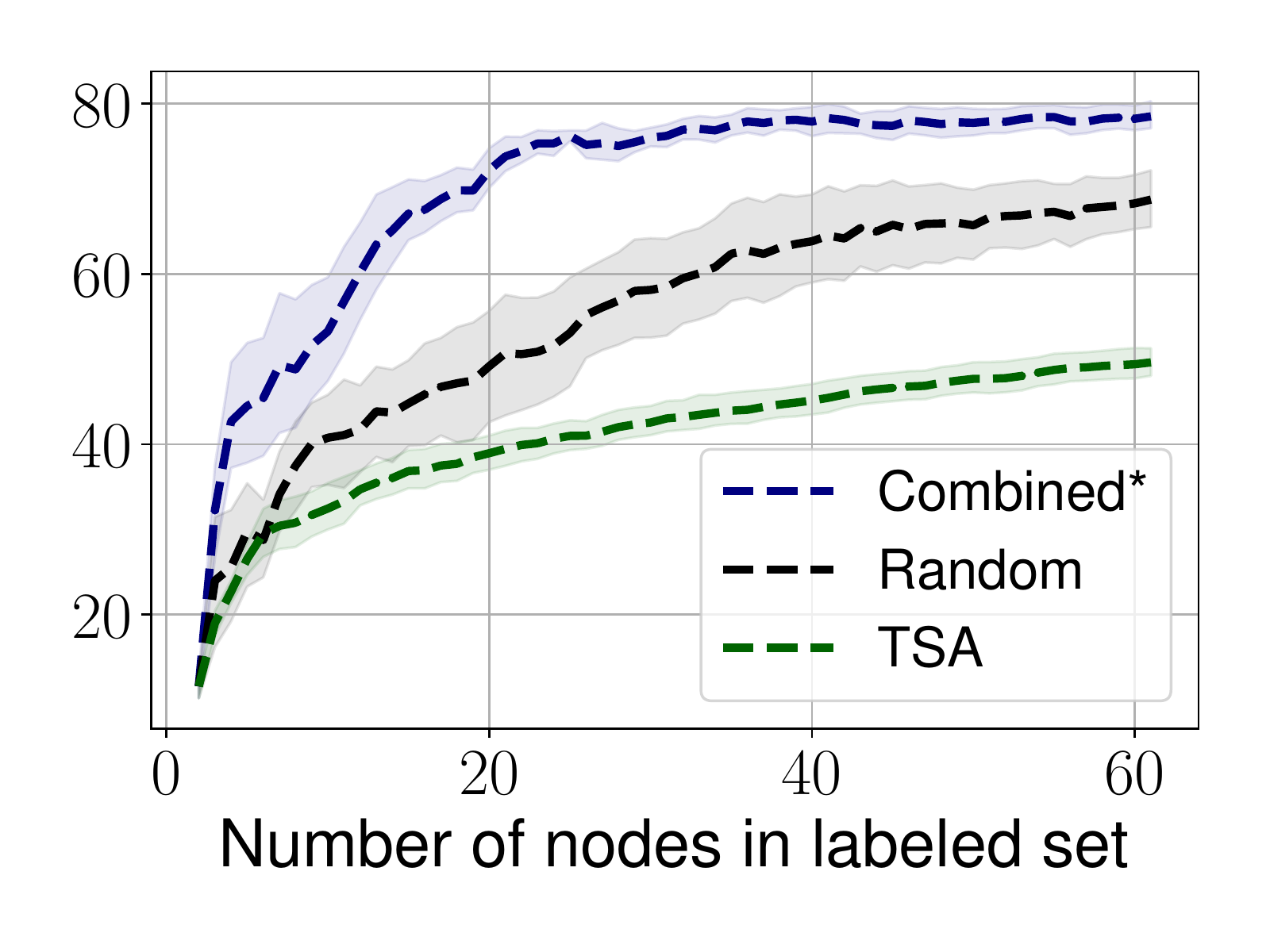}
    \vspace*{-5mm}\caption{Cora - inductive}
    \label{fig:coraindu}
    \end{subfigure}\hfill
\begin{subfigure}[t]{0.24\textwidth}
      \includegraphics[scale=0.27]{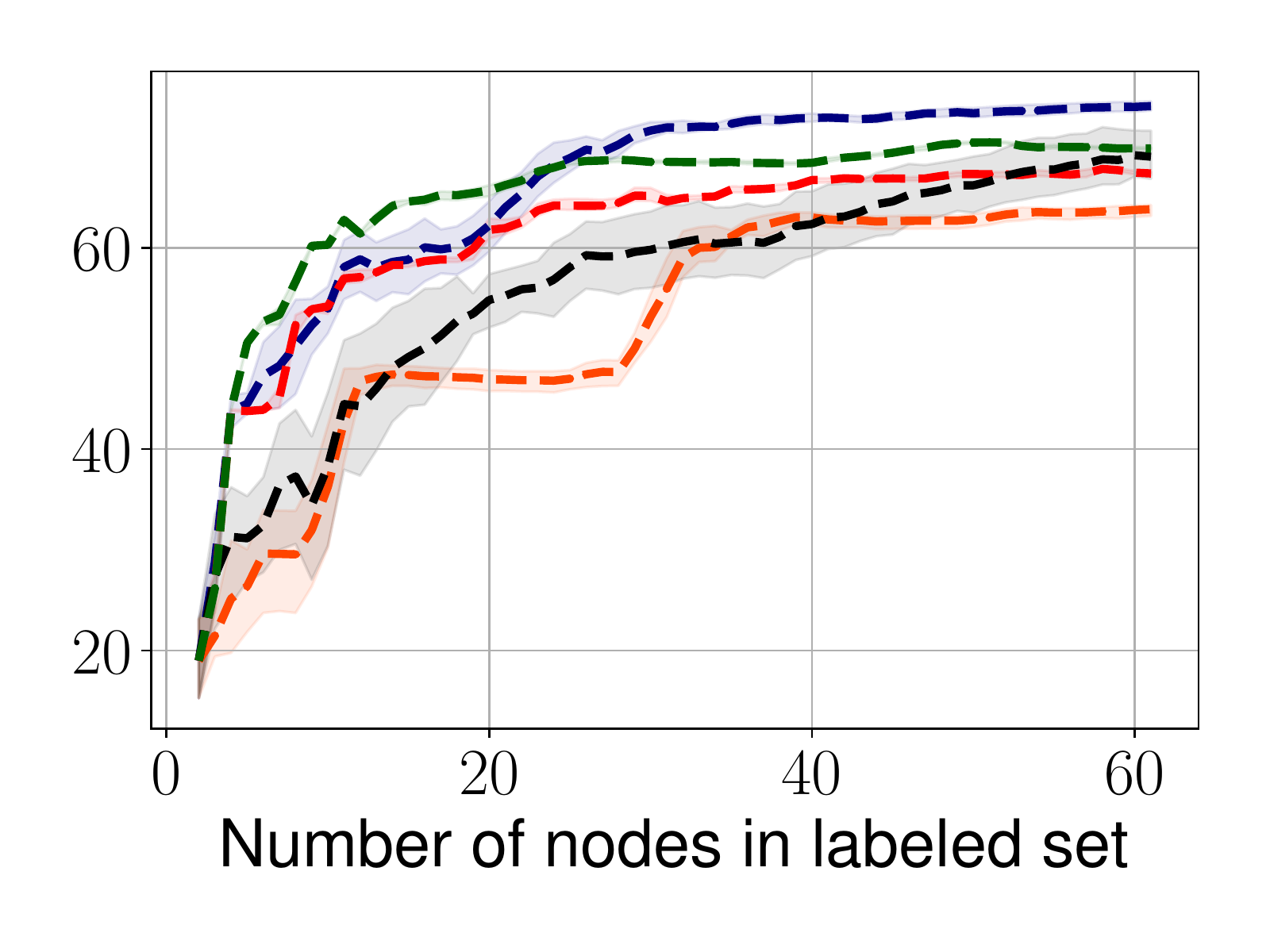}
    \vspace*{-5mm}\caption{Citeseer - transductive}
    \label{fig:webKBtransdu}
    \end{subfigure}\hfill
\begin{subfigure}[t]{0.24\textwidth}
     \includegraphics[scale=0.27]{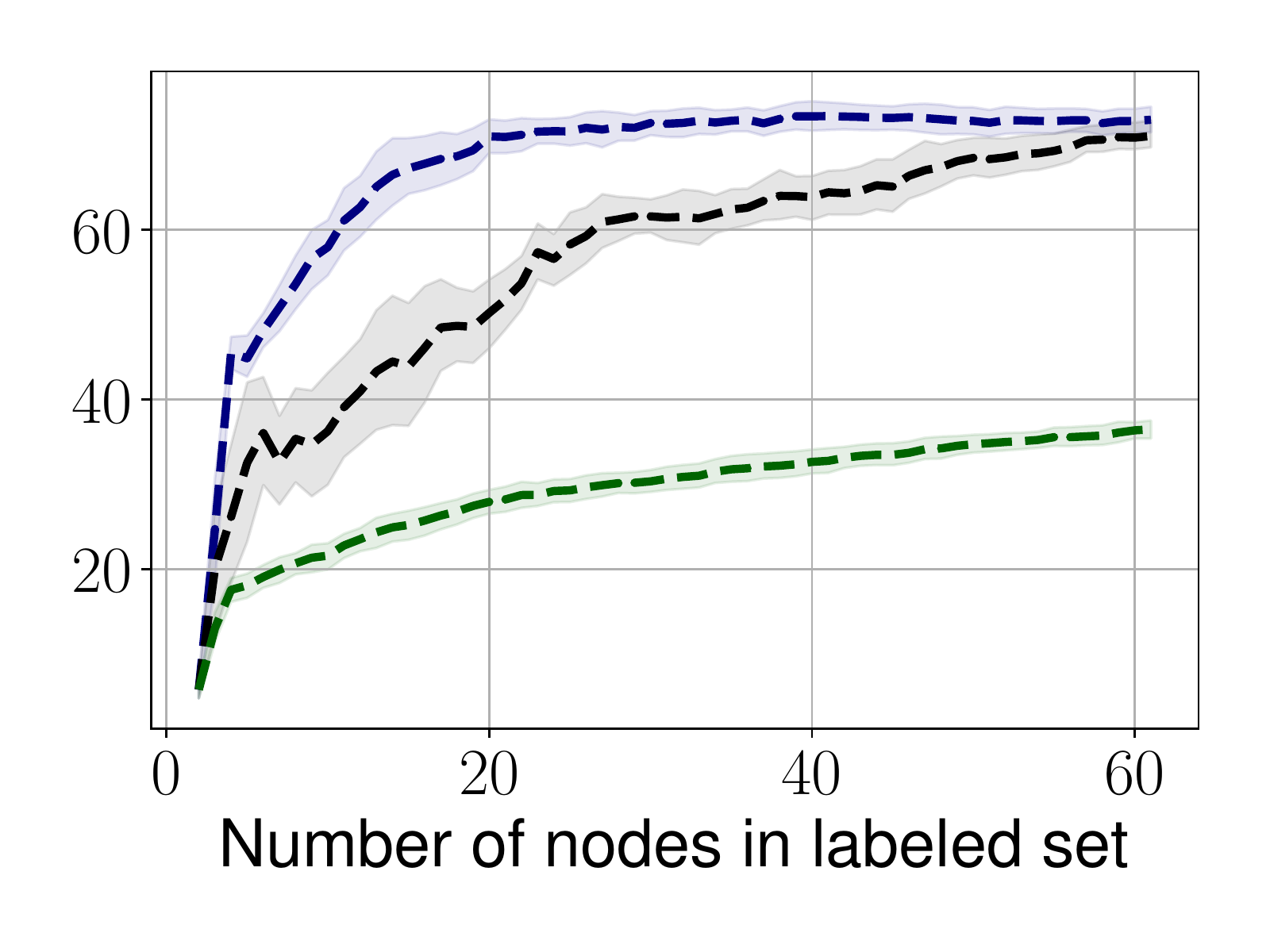}
    \vspace*{-5mm}\caption{Citeseer - inductive}
    \label{fig:webKBindu}
     \end{subfigure}\hfill
   
    \caption{Performance comparison between the label propagation algorithms and
      the proposed combined model-averaging expected error
      minimization method for the case when the initial label set
      consists of one random node. In the transductive
    setting, accuracy is evaluated across all unlabelled nodes; in the
  inductive setting, accuracy is evaluated on a held-out test set of
  nodes.}
\label{fig:indu_transdu}

\end{figure*}

{\em Experiment 1: Initial Labelled Set, Transductive}: Each algorithm
is initially provided with a small set of randomly chosen labelled
nodes. We evaluate performance on a set of test nodes comprising
20$\%$ of the unlabelled set. The algorithms cannot query nodes from
this evaluation set. Algorithms have access to the entire topology and
all node features. For the Cora and Citeseer datasets, we start with
$0.5\%$ of labelled nodes. We reduce this to $0.01\%$ for the larger
datasets to achieve similar initial set sizes. 

\begin{figure}
    \centering
    \includegraphics[scale=0.35]{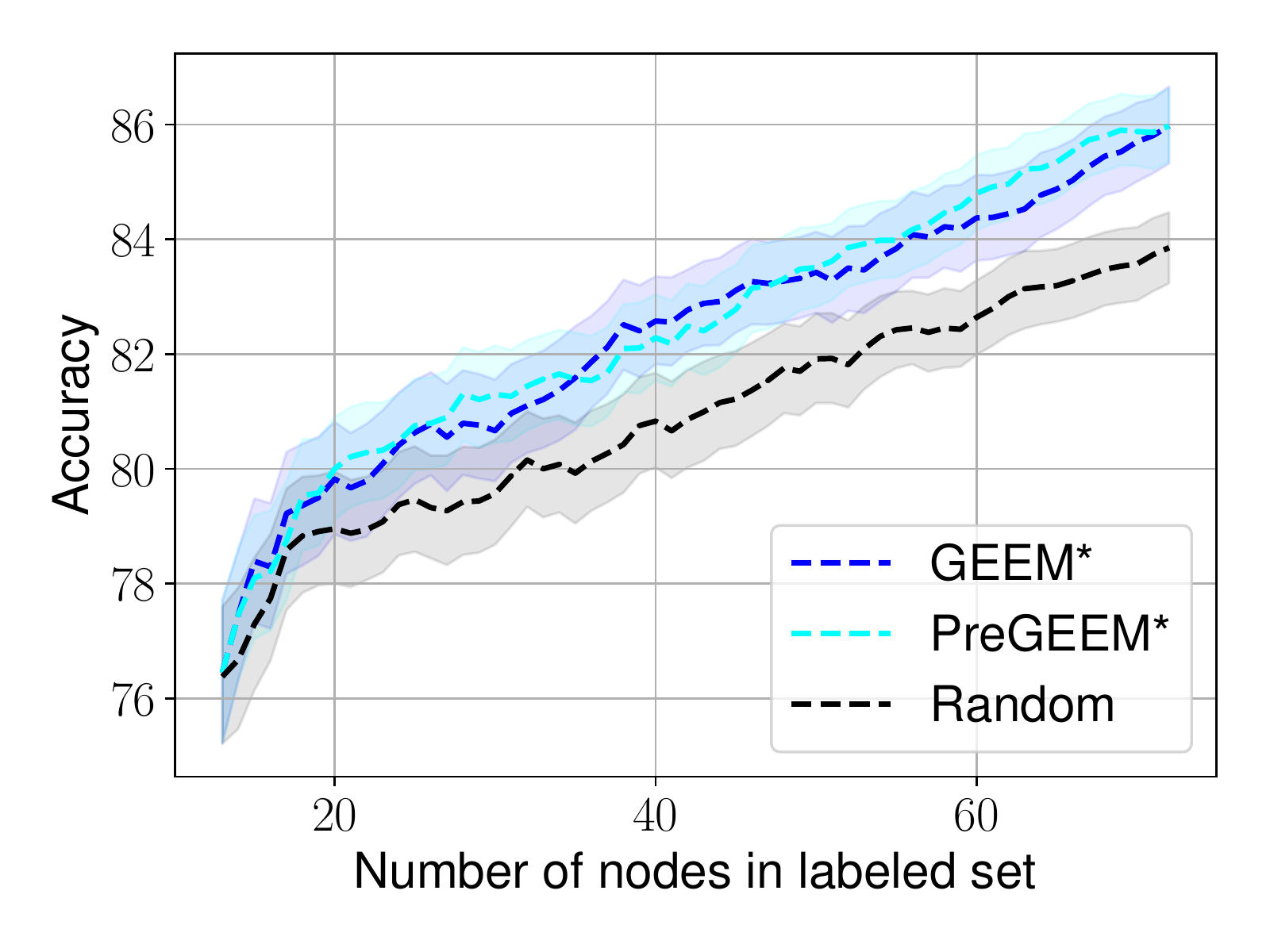}
    \caption{Performance of the active learning algorithms for
      detection of faulty links in the microwave link network.}
    \label{fig:prac}
\end{figure}

{\em Experiment 2: Single Labelled Node, Transductive and Inductive}:
Algorithms start with a single random labeled node. We examine two settings. In the
transductive setting, algorithms know the entire graph and can access
all features. Performance is assessed over all unlabelled
nodes. In the inductive setting, a portion of the graph is held out
for testing; the algorithms do not have access to the features and topology
information for these nodes.

\subsection{Results and Discussion} {\em Experiment 1}:
Figure~\ref{fig:res} and Table~\ref{table:acc} show how the
algorithms' accuracies change as more nodes are added to an initial
labelled set of size 10-20 nodes. Label propagation algorithm
performances are not shown because they are outperformed by the GCN
methods for this scenario. For all presented datasets, the proposed
algorithms outperform the other GCN-based methods. This holds even for
the cases when the hyperparameters of the GCN have been optimized
using a validation set. When the hyperparameters are not tuned, the
performance of the AGE algorithm deteriorates dramatically. It is
better to choose the query node randomly.  AGE outperforms the ANRMAB
algorithm for the datasets where its weight-adapting parameter was
tuned (Cora, Citeseer and Pubmed). The Random baseline and the proposed GEEM method use the same
classifier, so they differ only in the nodes that are queried.
Choosing an informative set of nodes using our proposed methods leads
to a substantially improved accuracy in all cases. For the Cora
dataset, the optimized GCN classifier initially outperforms the
SGC model. However GEEM quickly
outperforms as more nodes are queried, showing that the selection
algorithm is more effective.  Comparing PreGEEM performance against
GEEM, it is clear that the approximation has very little impact; there
is no clear performance difference between the two. 
\begin{table}[h!]
  \caption{Experiment 1 and Practical Application: Average accuracy at
    different budgets. Asterisks indicate that a Wilcoxon ranking test
    showed a significant difference (at $5\%$ significance level)
    between the marked method and the best performing baseline.}\label{table:acc}
\vspace{0.1cm}
\small{\begin{center}
\begin{tabular}{lccccc} \midrule[0.25ex]
\textbf{budget $b$} & \textbf{0}    & \textbf{1}    & \textbf{10}    & \textbf{30}    &    $b$ \\  \midrule[0.25ex]
\multicolumn{5}{c}{\textbf{Cora}}           & (60)  \\  \midrule
	\textbf{GEEM*}                                                    & 39.6          & 46.5  &       \textbf{69.8*} & \textbf{77.2*} & 79.9          \\
\textbf{PreGEEM*}                                                 & 39.6          & 46.5          & 68.2*          & 77.1*          & \textbf{80.3} \\
\textbf{Random}    & 39.6 &  40.2  & 49.7  & 63.0  & 73.3      \\
\textbf{AGE}                                                     & \textbf{46.6} & \textbf{52.7} & 61.6           & 74.9           & 79.8          \\
\textbf{ANRMAB}                                                     & \textbf{46.6} & 47.5          & 59.1           & 72.7           & 78.1          \\ \hline
 \multicolumn{5}{c}{\textbf{Citeseer}}           & (60)   \\ \midrule
	\textbf{GEEM*}                                                    & 40.5          & \textbf{49.7*}        & 65.8*         & 71.2           & 72.8          \\
\textbf{PreGEEM*}  & 40.5          & \textbf{49.7*}         & \textbf{66.5*}      & \textbf{71.8 }          & \textbf{73.3 }         \\
\textbf{Random}    & 40.5 &  44.1  & 53.8  & 64.4  & 70.4    \\
\textbf{AGE}                                                     &\textbf{ 41.2 }         & 44.7         & 60.5          & 69.1           & 71.4          \\
\textbf{ANRMAB}                                                     & \textbf{41.2 }         & 44.1         & 55.7          & 64.6          & 69.4         \\ \hline

 \multicolumn{5}{c}{\textbf{Pubmed}}           & (40)   \\ \midrule
\textbf{GEEM*}    &  52.3 & 58.1 & \textbf{72.6 } & \textbf{77.6}  & \textbf{78.7 } \\
\textbf{PreGEEM*}  & 52.3 & 58.1 & 69.3  & 77.2  & 78.0 \\
\textbf{Random}    & 52.3 & 54.1 & 64.7 & 72.3 & 73.9 \\
\textbf{AGE}   &   \textbf{57.3} & \textbf{60.8} & 70.4  & 76.7  & 78.1  \\
\textbf{ANRMAB}   & \textbf{ 57.3} & 58.8 & 69.5  & 74.1 & 75.7 \\ \hline

\multicolumn{5}{c}{\textbf{Am-photo}}           & (40) \\ \midrule
\textbf{GEEM*}                                                    & \textbf{59.6*}         &\textbf{ 64.3*  }        & \textbf{82.4*}           & \textbf{89.2*}           & \textbf{90.7*}        \\
\textbf{PreGEEM*}                                                 & \textbf{59.6}*          & \textbf{64.3*}         & 80.3*         & 88.8*         & 89.6         \\
\textbf{Random}    & \textbf{59.6*} &  61.4* &  72.0  & 82.3  & 87.6   \\
\textbf{AGE}                                                     & 45.5        & 52.0         & 51.5          & 67.8         & 69.3         \\
\textbf{ANRMAB}                                                     & 45.5        & 50.6         & 62.6          & 67.8         & 70.0         \\\hline
\multicolumn{5}{c}{\textbf{Am-comp.}}           & (40) \\ \midrule

\textbf{GEEM*}                                                    & \textbf{54.6*}          & \textbf{59.8*}          & \textbf{68.8*}           & 74.8*           & 76.8*         \\
\textbf{PreGEEM*}                                                 &\textbf{54.6* }       & \textbf{59.8*}          & 68.4*         & \textbf{76.5*}           & \textbf{77.5*}          \\

\textbf{Random}    & \textbf{54.6*} &  57.7* &  65.9 &  72.8 &  73.3 \\
\textbf{AGE}                                                     & 47.1        & 41.5         & 51.6          & 52.4          & 53.3         \\
\textbf{ANRMAB}                                                     & 47.1         & 49.4         & 54.6          & 58.7          & 58.5        \\\hline
 \multicolumn{5}{c}{\textbf{Microwave}}           & (60) \\ \midrule
\textbf{GEEM*}          &  \textbf{76.4 }  & \textbf{77.5}      & 80.1& \textbf{82.9*} & \textbf{86.0*}    \\
\textbf{PreGEEM*}    & \textbf{76.4 }  & \textbf{77.5 }     & \textbf{80.3} & 82.4* & \textbf{86.0*}   \\
\textbf{Random}                                                     & \textbf{76.4 }     & 76.7       & 79.1        & 81.0         & 83.9       \\
\textbf{AGE}                                                     & 69.1       & 68.3        & 70.3          & 70.3          & 75.1         \\
\textbf{ANRMAB}    & 69.1  & 67.2   &  72.3 & 73.5  &    73.2       \\
\hline
\end{tabular}
\end{center}}
\end{table}

{\em Experiment 2}: Figure~\ref{fig:indu_transdu} compares the
performance of the proposed Combined method with the label propagation
algorithms. In the transductive setting, the proposed method is much
better than Random selection. Since it incorporates the TSA technique, its
performance is similar to TSA  when few nodes have been queried. As
the number of labels increases, there starts to be a small but
significant improvement in accuracy. The inability of the label
propagation methods to adapt to the inductive setting is shown clearly
in Figures~\ref{fig:indu_transdu}(b) and~\ref{fig:indu_transdu}(d). In order to choose effective nodes
to query, these methods need to know the topology of the entire
graph. By contrast, the Combined method, which incorporates graph-based logistic
regression, achieves similar performance in both inductive and
transductive settings. 

\subsection{Practical Application}
\label{sec:prac}
To give a concrete motivating application for PreGEEM, we also report
the results of experiments on a private company dataset obtained from
measurements of a microwave link network. Currently, faulty links are
identified by human operators who must process lengthy performance log
files. The identification or labelling of a faulty link takes a few
minutes. Link performances vary substantially over time, so it is
necessary to repeatedly label data. It is desirable to automate the
faulty link detection procedure by training a classifier. Active
learning has the potential to substantially reduce the time an
operator must devote to the labelling task each week. For graphs the
size of common microwave link networks, the GEEM algorithm can return
a query in approximately one to two minutes, so this is an example
where the PreGEEM algorithm can compute the next query during the
labelling process.

The graph is constructed directly from the
physical topology and is important because graph-based classification
significantly outperforms classification algorithms that ignore the
network.  The features are link characteristics such as received
signal strength and signal distortion
metrics. Table~\ref{table:dataset_statistics} provides the statistics
of the dataset. We consider an experiment where an initial labelled set of 8 links is
provided, and the active learning algorithm must identify query
nodes. Table~\ref{table:acc} and Figure~\ref{fig:prac} compare the
performance of GEEM, PreGEEM, Random, AGE and ANRMAB. AGE and ANRMAB performs much worse than random selection because the GCN is
inaccurate for a small number of labels. GEEM and PreGEEM achieve a
small but significant improvement. 

\section{Conclusion}

We have introduced an active learning algorithm for node
classification on attributed graphs that uses SGC (a linearized GCN) in an expected
error minimization framework. Numerical experiments demonstrate that
the proposed method significantly outperforms existing active learning
algorithms for attributed graphs without relying on a validation set.  We also proposed a
preemptive algorithm that can generate a query while the oracle is
labelling the previous query, and showed experimentally that this
approximation does not impact the performance.

\bibliographystyle{icml2020}
\bibliography{reference}

\begin{thebibliography}{27}
\providecommand{\natexlab}[1]{#1}
\providecommand{\url}[1]{\texttt{#1}}
\expandafter\ifx\csname urlstyle\endcsname\relax
  \providecommand{\doi}[1]{doi: #1}\else
  \providecommand{\doi}{doi: \begingroup \urlstyle{rm}\Url}\fi

\bibitem[Berberidis \& Giannakis(2018)Berberidis and Giannakis]{berberidis2018}
Berberidis, D. and Giannakis, G.~B.
\newblock Data-adaptive active sampling for efficient graph-cognizant
  classification.
\newblock \emph{IEEE Trans. Signal Processing}, 66:\penalty0 5167--5179, Oct.
  2018.

\bibitem[Cai et~al.(2017)Cai, Zheng, and Chang]{cai2017}
Cai, H., Zheng, V.~W., and Chang, K.~C.
\newblock Active learning for graph embedding.
\newblock \emph{arXiv preprint arXiv:1705.05085}, 2017.

\bibitem[Defferrard et~al.(2016)Defferrard, Bresson, and
  Vandergheynst]{defferrard2016}
Defferrard, M., Bresson, X., and Vandergheynst, P.
\newblock Convolutional neural networks on graphs with fast localized spectral
  filtering.
\newblock In \emph{Proc. Advances in Neural Information Processing Systems},
  pp.\  3844--3852, Barcelona, Spain, Dec. 2016.

\bibitem[Gal et~al.(2017)Gal, Islam, and Ghahramani]{gal2017}
Gal, Y., Islam, R., and Ghahramani, Z.
\newblock Deep bayesian active learning with image data.
\newblock In \emph{Proc. Int. Conf. on Machine Learning}, pp.\  1183–1192,
  Sydney, Australia, Aug. 2017.

\bibitem[Gao et~al.(2018{\natexlab{a}})Gao, Wang, and Ji]{gao2018b}
Gao, H., Wang, Z., and Ji, S.
\newblock Large-scale learnable graph convolutional networks.
\newblock In \emph{Proc. Int. Conf. on Knowledge Discovery \& Data Mining},
  pp.\  1416–1424, London, United Kingdom, Aug. 2018{\natexlab{a}}.

\bibitem[Gao et~al.(2018{\natexlab{b}})Gao, Yang, Zhou, Wu, Pan, and
  Hu]{gao2018a}
Gao, L., Yang, H., Zhou, C., Wu, J., Pan, S., and Hu, Y.
\newblock Active discriminative network representation learning.
\newblock In \emph{Proc. Int. Joint Conf. Artificial Intelligence}, pp.\
  2142--2148, Stockholm, Sweden, Jul. 2018{\natexlab{b}}.

\bibitem[Hamilton et~al.(2017)Hamilton, Ying, and Leskovec]{hamilton2017}
Hamilton, W., Ying, Z., and Leskovec, J.
\newblock Inductive representation learning on large graphs.
\newblock In \emph{Proc. Advances in Neural Information Processing Systems},
  pp.\  1024--1034, Long Beach, CA, US, Dec. 2017.

\bibitem[Hoi et~al.(2006)Hoi, Jin, Zhu, and Lyu]{hoi2006}
Hoi, S. C.~H., Jin, R., Zhu, J., and Lyu, M.~R.
\newblock Batch mode active learning and its application to medical image
  classification.
\newblock In \emph{Proc. Int. Conf. on Machine Learning}, pp.\  417--424,
  Pittsburgh, PA, US, Jun. 2006.

\bibitem[Ji \& Han(2012)Ji and Han]{ji2012}
Ji, M. and Han, J.
\newblock A variance minimization criterion to active learning on graphs.
\newblock In \emph{Proc. Int. Conf. Artificial Intelligence and Statistics},
  pp.\  556--564, La Palma, Canary Islands, Apr. 2012.

\bibitem[{Jun} \& {Nowak}(2016){Jun} and {Nowak}]{jun2016}
{Jun}, K. and {Nowak}, R.
\newblock Graph-based active learning: A new look at expected error
  minimization.
\newblock In \emph{Proc. IEEE Global Conf. Signal and Information Proc.}, pp.\
  1325--1329, Dec. 2016.

\bibitem[Kipf \& Welling(2017)Kipf and Welling]{kipf2017}
Kipf, T. and Welling, M.
\newblock Semi-supervised classification with graph convolutional networks.
\newblock In \emph{Proc. Int. Conf. Learning Representations}, Toulon, France,
  Apr. 2017.

\bibitem[Kurzendorfer et~al.(2017)Kurzendorfer, Fischer, Mirshahzadeh, Pohl,
  Brost, Steidl, and Maier]{kurzendorfer2017}
Kurzendorfer, T., Fischer, P., Mirshahzadeh, N., Pohl, T., Brost, A., Steidl,
  S., and Maier, A.
\newblock Rapid interactive and intuitive segmentation of 3d medical images
  using radial basis function interpolation.
\newblock \emph{Journal of Imaging}, 3:\penalty0 56, Nov. 2017.

\bibitem[Liu et~al.(2019)Liu, Chen, Li, Zhou, Li, and Song]{liu2019}
Liu, Z., Chen, C., Li, L., Zhou, J., Li, X., and Song, L.
\newblock Geniepath: Graph neural networks with adaptive receptive paths.
\newblock In \emph{Proc. AAAI Conf. on Artificial Intelligence}, pp.\
  4424--4431, Honolulu, HI, US, Jan. 2019.

\bibitem[Ma et~al.(2013)Ma, Garnett, and Schneider]{ma2013}
Ma, Y., Garnett, R., and Schneider, J.
\newblock $\sigma$-optimality for active learning on {G}aussian random fields.
\newblock In \emph{Proc. Adv. Neural Inf. Proc. Systems}, pp.\  2751--2759,
  Lake Tahoe, NV, US, Dec. 2013.

\bibitem[Minka(2002)]{minka2002}
Minka, T.
\newblock Bayesian model averaging is not model combination.
\newblock {MIT} Media Lab Note, 2002.

\bibitem[Namata et~al.(2012)Namata, London, Getoor, and Huang]{namata2012}
Namata, G., London, B., Getoor, L., and Huang, B.
\newblock Query-driven active surveying for collective classification.
\newblock In \emph{Proc. Workshop on Mining and Learning with Graphs, Int.
  Conf. Machine Learning}, 2012.

\bibitem[Parisot et~al.(2018)Parisot, Ktena, Ferrante, Lee, Guerrero, Glocker,
  and Rueckert]{parisot2018}
Parisot, S., Ktena, S.~I., Ferrante, E., Lee, M., Guerrero, R., Glocker, B.,
  and Rueckert, D.
\newblock Disease prediction using graph convolutional networks: Application to
  autism spectrum disorder and {A}lzheimer's disease.
\newblock \emph{Medical Image Analysis}, 48:\penalty0 117--130, Aug. 2018.

\bibitem[Roy \& McCallum(2001)Roy and McCallum]{roy2001}
Roy, N. and McCallum, A.
\newblock Toward optimal active learning through sampling estimation of error
  reduction.
\newblock In \emph{Proc. Int. Conf. on Machine Learning}, pp.\  441--448, San
  Francisco, CA, USA, June 2001.

\bibitem[Sen et~al.(2008)Sen, Namata, Bilgic, Getoor, Galligher, and
  Eliassi-Rad]{sen2008}
Sen, P., Namata, G., Bilgic, M., Getoor, L., Galligher, B., and Eliassi-Rad, T.
\newblock Collective classification in network data.
\newblock \emph{AI Magazine}, 29\penalty0 (3):\penalty0 93, Sep. 2008.

\bibitem[Settles(2009)]{settles2009}
Settles, B.
\newblock Active learning literature survey.
\newblock Computer Sciences Technical Report 1648, University of
  Wisconsin--Madison, 2009.

\bibitem[Shchur et~al.(2018)Shchur, Mumme, Bojchevski, and
  G{\"u}nnemann]{shchur2018}
Shchur, O., Mumme, M., Bojchevski, A., and G{\"u}nnemann, S.
\newblock Pitfalls of graph neural network evaluation.
\newblock In \emph{Relational Representation Learning Workshop, NeurIPS 2018},
  Montr{\'e}al, Canada, Dec. 2018.

\bibitem[Sivan et~al.(2019)Sivan, Gabel, and Schuster]{sivan2019}
Sivan, H., Gabel, M., and Schuster, A.
\newblock Online linear models for edge computing.
\newblock In \emph{Proc. Eur. Conf. Machine Learning and Principles and
  Practice of Knowledge Discovery in Databases (ECML-PKDD)}, 2019.

\bibitem[Veli{\v{c}}kovi{\'c} et~al.(2018)Veli{\v{c}}kovi{\'c}, Cucurull,
  Casanova, Romero, Li{\`o}, and Bengio]{velivckovic2018}
Veli{\v{c}}kovi{\'c}, P., Cucurull, G., Casanova, A., Romero, A., Li{\`o}, P.,
  and Bengio, Y.
\newblock Graph attention networks.
\newblock In \emph{Proc. Int. Conf. Learning Representations}, Vancouver,
  Canada, Apr. 2018.

\bibitem[Wu et~al.(2019)Wu, Souza, Zhang, Fifty, Yu, and Weinberger]{wu2019}
Wu, F., Souza, A., Zhang, T., Fifty, C., Yu, T., and Weinberger, K.
\newblock Simplifying graph convolutional networks.
\newblock In \emph{Proc. Int. Conf. Machine Learning}, pp.\  6861--6871, Long
  Beach, CA, US, Jun. 2019.

\bibitem[Zhu et~al.(2003{\natexlab{a}})Zhu, Ghahramani, and Lafferty]{zhu2003a}
Zhu, X., Ghahramani, Z., and Lafferty, J.
\newblock Semi-supervised learning using gaussian fields and harmonic
  functions.
\newblock In \emph{Proc. Int. Conf. on Machine Learning}, pp.\  912--919,
  Washington, DC, USA, Aug. 2003{\natexlab{a}}.

\bibitem[Zhu et~al.(2003{\natexlab{b}})Zhu, Lafferty, and Ghahramani]{zhu2003b}
Zhu, X., Lafferty, J., and Ghahramani, Z.
\newblock Combining active learning and semi-supervised learning using
  {G}aussian fields and harmonic functions.
\newblock In \emph{Proc. Workshop on The Continuum from Labeled to Unlabeled
  Data (ICML)}, pp.\  58--65, Washington, DC, US, Aug. 2003{\natexlab{b}}.

\bibitem[Zhuang \& Ma(2018)Zhuang and Ma]{zhuang2018}
Zhuang, C. and Ma, Q.
\newblock Dual graph convolutional networks for graph-based semi-supervised
  classification.
\newblock In \emph{Proc. Int. World Wide Web Conf.}, pp.\  499--508, Lyon,
  France, Apr. 2018.

\end{thebibliography}

\end{document}